\documentclass[11pt]{article}%

\usepackage{times}

\usepackage[a4paper, left=3cm, right=3cm, top=3cm, bottom=3cm]{geometry}%

\usepackage[utf8x]{inputenc}%
\usepackage[english]{babel}
\usepackage{amsmath}
\usepackage{amsfonts}
\usepackage{amssymb}
\usepackage{graphicx}
\usepackage{dsfont}
\usepackage{amsthm}
\usepackage{subcaption}
\usepackage[colorinlistoftodos]{todonotes}

\usepackage{hyperref}

\newcommand{\R}{\mathds{R}}
\newcommand{\E}{\mathds{E}}

\newcommand{\bC}{{\boldsymbol C}}

\newcommand{\N}{\mathds{N}}
\newcommand{\Z}{\mathds{Z}}
\newcommand{\dP}{\mathds{P}}
\newcommand{\dL}{\mathds{L}}

\newcommand{\cF}{\mathcal{F}}
\newcommand{\cL}{\mathcal{L}}
\newcommand{\cN}{\mathcal{N}}

\newcommand{\cA}{ \mathcal{A} }

\newcommand{\cU}{\mathcal{U}}
\newcommand{\cI}{\mathcal{I}}

\newcommand{\dd}{\mathrm{d}}

\newcommand{\dt}{\dd t}

\newcommand{\dW}{\dd W}
\newcommand{\dX}{\dd X}

\newcommand{\optA}{\boldsymbol{A}_0}
\newcommand{\ha}{\widehat{\boldsymbol{A}}}
\newcommand{\hA}{\widehat{\boldsymbol{A}}_{MLE}}
\newcommand{\hAl}{\ha}

\newcommand{\hAla}{\ha_{ad.}}

\newcommand{\mat}[1]{{\boldsymbol #1}}
\newcommand{\vect}[1]{\MakeLowercase{#1}}

\newcommand{\mA}{{\boldsymbol A}}
\newcommand{\mB}{\boldsymbol{B}}
\newcommand{\mC}{\boldsymbol{C}}
\newcommand{\mD}{\boldsymbol{D}}
\newcommand{\mU}{{\boldsymbol{U}}}
\newcommand{\mV}{{\boldsymbol{V}}}
\newcommand{\mM}{{\boldsymbol{M}}}

\newcommand{\mG}{{\boldsymbol{G}}}
\newcommand{\mSigma}{{\boldsymbol{\Sigma}}}

\newcommand{\optcA}{\cA_0}

\DeclareMathOperator*{\argmin}{arg\,min}

\DeclareMathOperator*{\Sp}{Sp}
\DeclareMathOperator*{\adj}{adj}
\DeclareMathOperator*{\supp}{supp}

\newcommand{\eqL}{\overset{\rm d}=}
\newcommand{\vc}{\mathrm{vec}}

\newcommand{\Esp}[1]{\ensuremath{\mathds{E}\left[ #1 \right]}}
\newcommand{\Espq}[2]{\ensuremath{\mathds{E}_{#1}\left[ #2 \right]}}

\newcommand{\Prob}[1]{\ensuremath{\mathds{P}\left[ #1 \right]}}

\newcommand{\Var}[1]{\ensuremath{\mathds{V}\mathrm{ar}\left( #1 \right)}}

\newcommand{\Cov}[1]{\ensuremath{\mathds{C}\mathrm{ov}\left( #1 \right)}}
\newcommand\cov{\Cov}

\newcommand{\one}[1]{\ensuremath{\mathds{1}_{#1}}}

\newcommand{\id}{\ensuremath{{\boldsymbol I}}}

\newcommand{\tr}{{\rm tr \,}}
\newcommand{\diag}{{\rm diag \,}}

\newcommand{\diagII}{\hat{\delta}_T}

\newcommand{\Vinf}{\boldsymbol{C}_{\infty}}

\newcommand{\II}{\widehat {\boldsymbol{C}}_T}

\newcommand{\IW}{{\boldsymbol{\varepsilon}}_T}

\newcommand{\llangle}[1]{\left \langle #1 \right \rangle}

\newcommand{\sL}[1]{ \langle #1 \rangle_{L^2}}
\newcommand{\sFl}[1]{\left \langle #1 \right \rangle_{F}}
\newcommand{\sF}[1]{\langle #1 \rangle_{F}}

\newcommand{\nInflarge}[1]{\left\| #1 \right\|_{\infty}}

\newcommand{\nL}[1]{\lVert #1\rVert_{L^2}}
\newcommand{\nF}[1]{\| #1 \|_{F}}
\newcommand{\nT}[1]{\| #1 \|_{2}}
\newcommand{\nO}[1]{\| #1 \|_{1}}
\newcommand{\nZ}[1]{\| #1 \|_{0}}
\newcommand{\nQ}[1]{\| #1 \|_{q}}
\newcommand{\nInf}[1]{\| #1 \|_{\infty}}
\newcommand{\nOp}[1]{\| #1 \|_{\mathrm{op}}}

\newcommand{\OO}[1]{O\left(#1\right)}

\newcommand{\cNN}[1]{\cN \left( #1 \right)}

\newcommand{\beq}[2]{
\begin{equation}
\label{#1}
#2
\end{equation}
}

\allowdisplaybreaks 

\usepackage{mathtools}  
\mathtoolsset{showonlyrefs=true} 

\newtheorem{lemma}{Lemma}

\newtheorem{remark}{Remark}
\newtheorem{corollary}{Corollary}
\newtheorem{theo}{Theorem}
\numberwithin{equation}{section}

\author{Stéphane Gaïffas\footnote{CMAP, Ecole Polytechnique and CNRS, Universit\'e Paris Saclay, Route de Saclay, 91128 Palaiseau cedex, France. Email: stephane.gaiffas@polytechnique.edu.}, Gustaw Matulewicz\footnote{CMAP, Ecole Polytechnique and CNRS, Université Paris Saclay, Route de Saclay, 91128 Palaiseau cedex, France. Email: gustaw.matulewicz@polytechnique.edu. This work was funded jointly by \textit{Chaire Risques Financiers} of the \textit{Risk Foundation},  the {\it Finance for Energy Market Research Centre}, the Natixis Foundation for Quantitative Research, and the Data Science Initiative of Ecole polytechnique}}
\title{Sparse inference of the drift of a high-dimensional Ornstein-Uhlenbeck process}
\begin{document}

\maketitle

\begin{abstract}
Given the observation of a high-dimensional Ornstein-Uhlenbeck (OU) process in continuous time, we proceed to the inference of the drift parameter under a row-sparsity assumption.
Towards that aim, we consider the negative log-likelihood of the process, penalized by an $\ell^1$-penalization (Lasso and Adaptive Lasso). 
We provide both non-asymptotic and asymptotic results for this procedure, by means of a sharp oracle inequality, and a limit theorem in the long-time asymptotics, including asymptotic consistency for variable selection.
As a by-product, we point out the fact that for the Ornstein-Uhlenbeck process, one does not need an assumption of restricted eigenvalue type in order to derive fast rates for the Lasso, while it is well-known to be mandatory for linear regression for instance.
Numerical results illustrate the benefits of this penalized procedure compared to standard maximum likelihood approaches both on simulations and real-world financial data. \\


\end{abstract}


\section{Introduction}
\label{s:introduction}

The Ornstein-Uhlenbeck, also called mean-reverting diffusion process, describes a process which evolves following a deterministic linear part with an added Gaussian noise, similarly to a vector-autoregressive process in discrete time.
This model is ubiquitous in quantitative finance, for instance the one-dimensional version is used for modeling rates and is called the Vasicek model~\cite{hull:2009}.
In a multi-dimensional setting, it can be therefore used to describe systems with linear interactions perturbed by Gaussian noise, see Figure~\ref{f:example_trajectory_matrix} below.
Among many others, an example of application is inter-bank lending~\cite{carmona:2013,fouque:2013}, where lending is a flux of reserves and is proportional to the difference in reserves.
A natural question is therefore how to estimate the interaction structure from the observation of the process.
Unfortunately, the optimal solution based on the maximum likelihood estimator (MLE) is typically quite inaccurate in high-dimensional settings, because of the well-known curse of dimensionality, see for instance~\cite{buhlmann2011statistics}.
However, in real-world applications, the interaction structure is sparse: in the example mentioned above, banks have typically only a few lending partners \cite{gabrieli:2014,gabrieli:2015,brauning:15}, as the lending arrangements are typically done on a personal level.

\begin{figure}
\centering
\begin{subfigure}{0.55\textwidth}
\centering
\includegraphics[height=5cm]{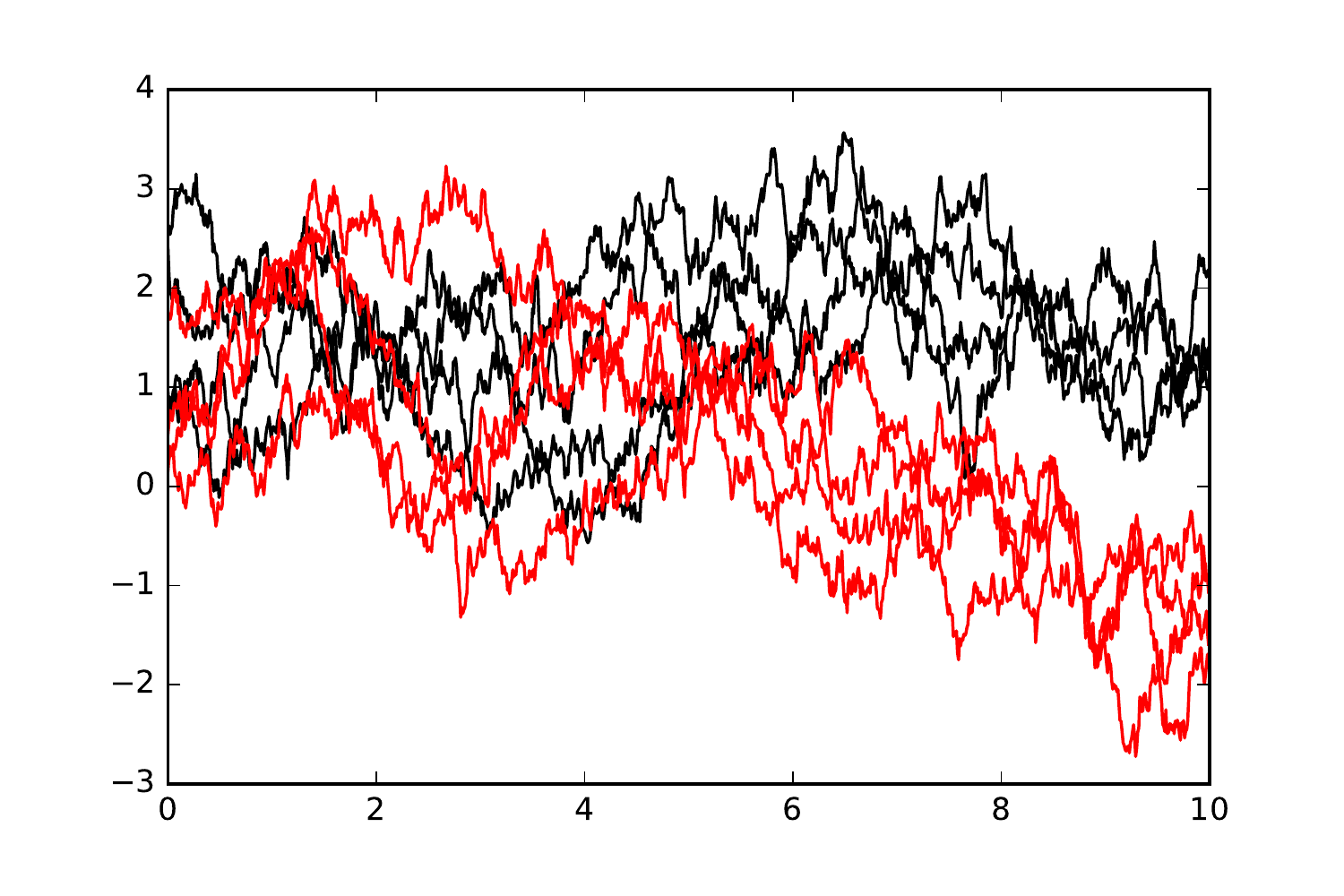}
\caption{Sample trajectory in 8 dimensions.}
\label{f:trajectory}
\end{subfigure}
\begin{subfigure}{0.4\textwidth}
\centering
\includegraphics[height=5cm]{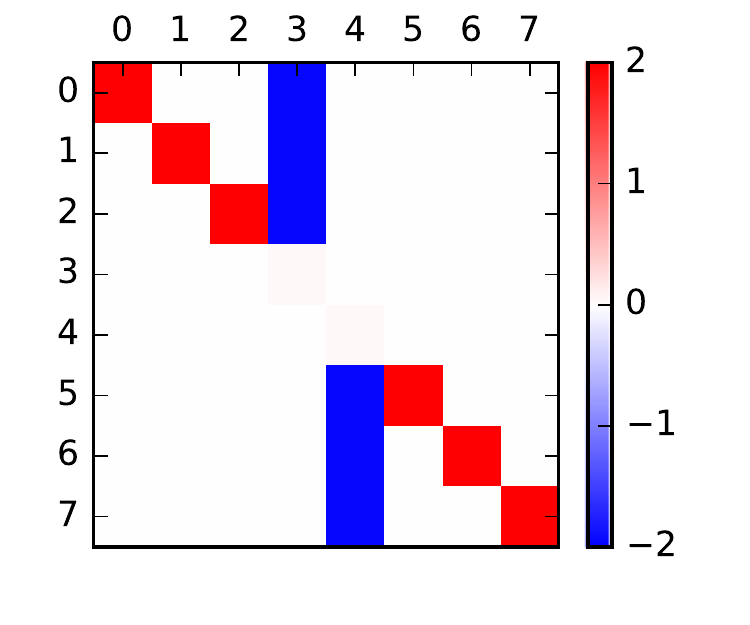}
\caption{Matrix parameter $\optA$.}
\label{f:matrix}
\end{subfigure}
\caption{On the right, heat-map representation of a sparse matrix $\optA$. 
In this particular example, the matrix is chosen in order to have two groups, $0$ to $3$ and $4$ to $7$, that are independent and tend to stay close within each group. On the left, plot of the 8 coordinates of the sample trajectory, each group being attributed a different color.
Our estimation procedure can be applied to find this kind of hidden structure from non-obvious trajectories.}
\label{f:example_trajectory_matrix}
\end{figure}

In this paper, we exploit this property by using a sparsity-inducing penalization.
Sparse inference using convex penalization has known a strong development in the past decade~\cite{buhlmann2011statistics,giraud2014introduction,bach2012optimization}, mostly for linear supervised learning.
Quite surprisingly, there is only a single previous attempt to this work to use these techniques in the setting of diffusion processes, in particular for the Ornstein-Uhlenbeck diffusion considered here, see~\cite{sokol:2013}, with no theoretical guarantees nor applications on real-world data.
The aim of this paper is therefore to fill this gap, and to give a complete theory for this case, by developing both non-asymptotic results by means of a sharp oracle inequality, see Section~\ref{s:nonasymptotic}, and asymptotic results (in the length of observation interval), see Section~\ref{s:asymptotic}, where we notably establish asymptotic consistency for selection of the support of $\mA_0$.
We also prove a minimax lower-bound for the problem of sparse inference in this model in Section~\ref{s:nonasymptotic}.
As a by-product, we exhibit a surprising fact in our analysis that for the Ornstein-Uhlenbeck process, one does not need to assume the restricted eigenvalue assumption~\cite{buhlmann2011statistics}, which is known to be mandatory for the linear regression model, see for instance~\cite{zhang2014lower}.

\subsection{Related work}
\label{ss:related}

We investigate in this article the question of recovery of the drift parameter of an Ornstein-Uhlenbeck process from the continuous observation of a single multidimensional trajectory on the 
interval $[0,T]$.
This relates to the much developed area of inference for stochastic processes in continuous time, see~\cite{kutoyants:2004statistical} for a survey on this topic. 
This work is also related to the field of high-dimensional statistics, in particular sparse inference, since we use a sparsity assumption on the parameter matrix, we refer to~\cite{buhlmann2011statistics,giraud:2014} for surveys on the topic.
Indeed, in this paper we study the Lasso~\cite{tibshirani1996regression} and Adaptive Lasso~\cite{zou:06} penalizations, 
applied to the multivariate Ornstein-Uhlenbeck process.

Note that, however, references that propose sparse inference techniques to stochastic processes are quite scarce.
A Vector Auto-Regressive (VAR) process can be seen as a discretization in time of an Ornstein-Uhlenbeck process, where $\optA$ is analogous to the VAR transition matrix.
The sparse estimation of a VAR process using a Lasso is the subject of \cite{basu:15}. 
However, our work differs on two fundamental points. The first relates to the graph structure implied by $\optA$. 
While~\cite{basu:15} assumes sparsity of the whole graph, we place the sparsity on a node level, restricting the maximum degree of the graph, since we work under a row-sparsity assumption, see Assumption~(H3) below. 
This prescribes for instance the existence of nodes which concentrate most connections, in line with observations of the interbank lending system, which note high connectedness only in the core of the network \cite{gabrieli:2015}.
The second relates to the continuous nature of the considered model.  Since the VAR model has finite dimension both in time and space~\cite{basu:15}, it is possible to analyze them jointly in a space of finite dimension equal to the product of the two dimensions. 
In this paper, we work in continuous time, which forces us to treat time and dimensionality in a fundamentally different way.
Another reference is~\cite{sokol:2013}, where the Lasso is considered as a strategy to estimate Ornstein-Uhlenbeck parameters in a sparse setting, but no theoretical results nor numerical experiments are provided for this problem.
Finally, we consider the particular notion of row-sparsity, which was considered previously for matrix estimation (with additive noise) in~\cite{klopp2015estimation}, instead of the full sparsity of $\optA$.

\subsection{The model, main assumptions and tools}
\label{ss:model}

Throughout the article we consider a $d$-dimensional Ornstein-Uhlenbeck process $X = (X_t)_{t \geq 0}$, where $X_t \in \R^d$ for any $t \geq 0$ is solution to the following stochastic differential equation 
\beq{def:OU}{
	\dX_t = - \optA X_t \dt + \dW_t, \quad \text{for any } t \geq 0,
}
where the initial value $X_0$ is given, $\optA$ is an unknown $d \times d$ matrix to be inferred, and $(W_t)_{t \geq 0}$ is a standard Brownian motion in $\R^d$ defined on a filtered space $( \Omega, \cF, (\cF_t)_{t \geq 0}, \dP)$.

We observe the process on an interval $[0, T]$ with $T > 0$. Based on the observation $(X_t)_{t \in [0, T]}$, we want to estimate $\optA$, under the assumption that $\optA$ has sparse rows, namely that a large number of their entries are zeros.
Note that $\optA^{ij} \not = 0$ encodes the fact that the trajectory of process~$j$ influences the dynamic of process~$i$, which is a property of particular interest for instance in interbank-lending as it implies lending activity from~$j$ to~$i$. Row-sparsity implies that each institution borrows from a limited number of institutions. 
More generally, in time-series analysis, it means that each trajectory is impacted by a limited number of other trajectories.

Throughout the paper, we work under the following assumptions.
\begin{description}
\item [(H1)] The spectrum of $\optA$ has strictly positive real parts.
\item [(H2)] $X_0$ follows the stationary distribution of the process.
\end{description}

These are standard assumptions for Ornstein-Uhlenbeck processes: Assumption~(H1) guarantees ergodicity of $(X_t)_{t \geq 0}$ and existence of a stationary distribution, and is necessary to ensure mean-reversion of the process, the real-world phenomenon that we want to capture and exploit in our modeling. 
Under~(H2) the process is stationary, which is interesting for two reasons. First, it simplifies the results as the initial position doesn't have to be treated differently from the rest of the trajectory. Second, in typical applications one assumes an equilibrium, hence stationarity. For example, in interbank lending there is no reason to assume that the first day of observation is any different from days that precede and follow.

Under these assumptions, the Ornstein-Uhlenbeck verifies interesting  properties.
For instance, we have 
\begin{equation}
	X_t \sim \cNN{0,\Vinf} \quad \text{with} \quad \Vinf = \Vinf(\mA) := \int_0^\infty e^{-\mA t} e^{-\mA^\top t} \dt,
\end{equation}
for all $t \geq 0$. 
For this and other classical properties we refer to~\cite{karatzas1991brownian}, see Section 5.6 herein.
In this model, the maximum-likelihood estimator (MLE) is given as the argument minimum of the following negative log-likelihood:
\beq{eq:logL}{
\cL_T(\mA) := \frac{1}{T} \int_0^T (\mA X_t)^\top \dX_t + \frac{1}{2T} \int_0^T (\mA X_t)^\top \mA X_t \dt,
}
and it can be written explicitly as 
\begin{equation}
\label{eq:mle_estimator}
	\hA := - \left( \int_0^T \dX_t X_t^\top \right) \left( \int_0^T X_t X_t^\top \dt \right) ^{-1}.
\end{equation}
The inverse exists almost surely as the integral is almost surely a symmetric positive definite matrix (see Section \ref{s:nonasymptotic} for more details).
The asymptotic normality of this MLE is well-know, indeed we have
\beq{eq:asymptotic_normality}{
	\sqrt{T} \left( \vc \ \hA - \vc \ \optA \right) \xrightarrow{\cL} \cNN{0,\Vinf^{-1} \otimes \id},
}
see \cite{jacod:01}, where $\xrightarrow{\cL}$ stands for convergence in distribution, $\id$ stands for the identity matrix in $\R^{d \times d}$, $\otimes$ is the matrix-Kronecker product and $\vc$ stands for the vectorization operator, that stacks rows of a $d \times d$ matrix into a flat vector with $d^2$ entries.

When $d$ is large, the performance of the MLE deteriorates, because of the curse of dimensionality problem, see \cite{buhlmann2011statistics} and our numerical results in Section \ref{ss:T_impact}.
So, as motivated above, we will reduce dimensionality using a sparsity-inducing penalization on this estimator, see Sections~\ref{s:nonasymptotic} and~\ref{s:asymptotic} below. 
Our analysis relies on the following two central quantities:
\begin{align*}
\II = \frac{1}{T} \int_0^T X_t X_t^\top \dt \quad \text{ and } \quad
\IW = \frac{1}{T} \int_0^T \dW_t X_t^\top.
\end{align*}x
The matrix $\II$ is the empirical covariance which satisfies
 $\E[\II] = \Vinf$. It is analogous to the Gram matrix in linear regression.
 The matrix $\IW$ is a noise term, note that $(t \mat{\varepsilon}_t)_{t \geq 0}$ is a martingale with quadratic variation given by $\llangle{ \vc \ t \mat{\varepsilon}_t} =
  t \widehat \bC_t \otimes \id$. 
 Using this matrix notation, we have for instance $\hA = \optA - \IW \II^{-1}$ and the matrix formulation of the problem
\beq{eq:logL_matrix}{
\cL_T(\mA) = \tr \mA^\top \IW + \frac{1}{2}(\mA - \optA) \II (\mA - \optA)^\top - \frac{1}{2} \optA \II \optA^\top. }

\paragraph{Notation.} 

For a matrix or a vector $x$, we denote by $\nQ{x}$ the entrywise $\ell_q$ norm for any $q \in [1, +\infty]$. 
The notation $\nZ{x}$ stands for the number of non-zero entries of $x$, $\nF{x} = \|x\|_2$ for the Frobenius norm of $x$ when it is a matrix; we consider also the Euclidean inner product $\sF{\mU, \mV} = \tr 
\mU^\top \mV$, where $\tr \mM$ is the trace of a matrix $\mM$ and define $\nOp{\mM}$ as the operator norm of $\mM$.
We also denote by $\sigma_{\min}(\mA)$ the smallest eigenvalue of a symmetric $\mA$, and $\diag(\mA)$ stands for the vector formed by the diagonal of $\mA$.
We also denote by $\supp(x)$ the support of $x$, i.e. the set of indices of the non-null coordinates of $x$, where $x$ is a matrix or a vector. 
Given a set of indices $\cI$, we denote by $x_{\mid \cI}$ the restriction of $x$ to the indices in $\cI$. 
Moreover, $\Sp \mA$ is the spectrum of $\mA$ and $\diag \mA$ is the extraction of the diagonal of $\mA$.
Additionally, we define
\begin{equation*}
	\nL{X}^2 = \frac{1}{T} \int_0^T |X_t|^2_2 \dt \quad \text{and} 
	\quad \sL{X,Y} = \frac{1}{T} \int_0^T X_t^\top Y_t \dt,
\end{equation*}
that correspond to the empirical norm and inner products along the observed trajectory of $(X_t)_{t \geq 0}$.

\subsection{Main results and organization of the paper} 
\label{sub:main_results_and_organization_of_the_paper}

In Section \ref{s:nonasymptotic} we introduce the Lasso estimator of $\optA$. Our main result, concerning non-asymptotic error bounds, is Theorem \ref{th:1norm_empirical_bound}. 
We show that this upper bound is asymptotically of the same order, up to logarithmic terms, as the lower bound we have in Theorem~\ref{th:sparse_lower_bound}. 
We conclude the section with Theorem~\ref{th:RE} which is an interesting by-product of the proof of Theorem~\ref{th:1norm_empirical_bound} and which states that a Restricted Eigenvalue condition is valid in our setting, when $\mA_0$ is symmetric. 
In Section~\ref{s:asymptotic} we introduce the Adaptive Lasso estimator and prove in Theorem \ref{th:adalasso_oracle} its asymptotic normality and support recovery properties. 
Numerical experiments are provided in Section \ref{s:numerical}, where we illustrate the benefits of sparse inference over direct maximum likelihood estimation. In Section~\ref{s:proofs} we provide the proofs of the properties from the preceding Sections.

\section{Non-asymptotic error bounds for Lasso}
\label{s:nonasymptotic}

Given a regularization parameter $\lambda >0$, we define the Lasso estimator by:
\beq{def:lasso}{
\hAl = \argmin_{\mA \in \R^{d \times d}} \cL_T(\mA) + \lambda \nO{\mA}.
}
The uniqueness of $\hAl$ derives from the strict convexity of $\cL_T$, which comes from the fact $\II$ is a.s.\ a positive definite matrix, see Equation \eqref{eq:logL_matrix}. 
Indeed, observe that for any $\vect{u} \in \R^d$, $\nT{\vect{u}} = 1$, we have $\vect{u}^\top \II \vect{u} = T^{-1} \int_0^T (\vect{u}^\top X_t)^2 \dt$ which can be zero only if the trajectory is included in a hyperplane of $\R^d$. The observation length $T>0$ is fixed in the whole Section. We also fix an integer $1 \leq s \leq d$ and express the sparsity of $\optA$ in the following assumption:

{\begin{description}
\item [{\bf (H3)}] The true parameter is row-$s$-sparse, i.e.
\beq{eq:optA_sparse}{
\nZ{\optA^{i,\bullet}} \leq s \; \text{ for all } i=1, \ldots, d,
}
where $\mA^{i,\bullet}$  stands for the vector such that for any $j \leq d$, $(\mA^{i,\bullet})^j = \mA^{ij}$ for any matrix $\mA$.
\end{description}}

This assumption notably differs from a sparsity assumption on the whole matrix parameter, but has already been used in matrix estimation, for instance in \cite{klopp:15} for additive noise.
We also need to introduce a technical hypothesis on the deviation of $\II$ from $\Vinf$.

\begin{description}
\item [(H4)] There exists a non-decreasing function $H$, positive on $\R^+$, such that for any vector $\vect{u}$ verifying $\nT{\vect{u}} \leq 1$, we have:

\beq{eq:hypo_exponential_deviation}{
\Prob{|\vect{u}^\top (\II - \Vinf) \vect{u}| \geq R } \leq 2 \exp (- T H(R)).
}
\end{description}

We actually prove this assumption in the case where $\optA$ is symmetric, see Theorem \ref{th:H4} in Section \ref{ss:deviation_hypothesis}. The proof is based on a concentration inequality for integrals of functionals of a stochastic process from \cite{cattiaux:2007}. Furthermore, the convergence of $\II$ to $\Vinf$ is constrained by the speed of decorrelation of the process, which is the slowest precisely for symmetric parameters $\optA$, see \cite{hwang:93}. 
We therefore conjecture Assumption~(H4) to hold also for a  non-symmetric~$\mA_0$.

The set of Assumptions~(H1) -- (H4) are relatively unrestrictive. 
As already explained, Assumption~(H1) is necessary for stationarity while Assumption~(H2) could be possibly eliminated, since the exponentially decreasing autocorrelation means that the distribution of $X$ is rapidly approaching the stationary distribution, but this would unnecessarily clutter our results. 
Assumption~(H4) is not very restrictive: as mentioned above it is proved for a symmetric $\mA_0$, and we conjecture it to be true in general (but were unable to prove the general case yet). 
Finally, Assumption~(H3) is the sparsity assumption assumed throughout the paper on $\optA$.

\begin{theo}
\label{th:1norm_empirical_bound}
Assume (H1) -- (H4). Set $\gamma >1$, $0 \leq \tau < \gamma -1$, $\epsilon_0 \in (0,1)$ and define
\beq{def:lambda}{
\lambda_T := \gamma \sqrt{ \frac{4 e |\diagII|_\infty}{T} \left( \frac{1}{2} \log \frac{2 \pi^2 d^2 }{3 \epsilon_0} + \log ( 2 + | \log (T \diagII) |_\infty ) \right) }
}
where $\diagII := \diag \II$ and $\log$ is applied entrywise on $T \hat \delta_T$. 
Set $c_0 := \frac{\gamma + \tau + 1}{\gamma - \tau -1}$, $\kappa := \sqrt{\frac{\min \Sp (\Vinf)}{2}}$ and assume that
\begin{equation*}
	T \geq T_0 := H \left(\frac{\kappa^2}{9 (c_0 + 2)^2}\right)^{-1} \left( s \log \Big(21d \wedge \frac{21 ed}{s} \Big) + \log \Big( 
	\frac{4}{\epsilon_0} \Big) \right).
\end{equation*}
Then, for any row-$s$-sparse matrix $\mA$, the lasso estimator $\hAl := \ha_{\lambda_T}$ verifies
\beq{eq:norm1_L2_err}{
2 \tau \gamma^{-1} \lambda_T \nO{\hAl - \mA} + \nL{ ( \hAl - \optA ) X }^2 \leq \nL{ ( \mA  - \optA ) X }^2 + \left( \frac{1 + \gamma + \tau}{\gamma \kappa} \right)^2 \lambda_T^2 ds}
with probability at least $ 1 - \epsilon_0$.
\end{theo}

The proof of Theorem~\ref{th:1norm_empirical_bound} is detailed in Section~\ref{ss:master_theorem} below.
It relies on a Restricted Eigenvalue property, see Theorem~\ref{th:RE} below, which we prove using Assumptions (H1)--(H4), as well as on a deviation property, see Theorem~\ref{th:deviation} from Section~\ref{ss:deviation} below.
Theorem~\ref{th:1norm_empirical_bound} provides a sharp oracle inequality, with leading constant~1 in front of the bias term $\nL{ ( \mA  - \optA ) X }^2$.
The penalization parameter $\lambda$ is a function of the observations through $\II$. 
However, the proof of Theorem~\ref{th:1norm_empirical_bound} uses Equation~\eqref{eq:re_dev} which states that in the same set of events of probability at least $1- \epsilon_0$, we have $\kappa^2 \leq \diagII^{i} \leq \Vinf^{ii} + \kappa^2$ for any $i=1, \ldots, n$.
We can therefore safely bound $\diagII$ from below and above by deterministic constants in the statement of Theorem~\ref{th:1norm_empirical_bound}.

The convergence rate obtained in Theorem~\ref{th:1norm_empirical_bound}  almost matches the minimax lower bound provided in Theorem~\ref{th:sparse_lower_bound} below. 
Indeed, the rate is $\lambda^2 d s$, up to numerical constants, and using the upper bound for $\diagII$ given above, we end up with a convergence rate of order
\begin{equation*}
	\frac{d s(\log d + \log \log T)}{T}.
\end{equation*}
Let us recall that $d s$ is the sparsity of $\optA$, under the row-sparsity~(H3).
The minimax lower bound from Theorem~\ref{th:sparse_lower_bound} is of order $ds \log(d / s) / T$. The only main difference is between the terms $d$ and $d / s$ within the logarithm, and the negligible poly-logarithmic term $\log \log T$.
We conjecture that an exact match (up to constants) between the upper and the minimax lower bound is possible, by considering ordered-$\ell_1$ penalization, also called SLOPE, see~\cite{su2016slope, bellec:16}, where such results are provided for linear regression only.
However, such a development is way beyond the current focus of this paper: the choice of the weights involved in SLOPE is a difficult task in the setting considered here.

The next corollary provides errors bounds on the parameter $\optA$ using different norms.

\begin{corollary}
\label{c:applied_inequalities}
With the same assumptions and notation as in Theorem \ref{th:1norm_empirical_bound}, the following holds with a probability larger than $1-\epsilon_0$:

\begin{enumerate}
\item for the empirical norm:

\beq{eq:empirical_bound}{
\nL{ ( \hAl - \optA ) X } \leq \frac{1 + \gamma}{\gamma \kappa} \lambda_T \sqrt{ds}
}

\item for the $\ell^1$ norm, with $\tau >0$:

\beq{eq:norm1_err}{
\nO{\hAl - \optA} \leq \frac{(1 + \tau + \gamma)^2 }{2 \gamma \tau \kappa^{2}} \lambda_T d s
}

\item for the Frobenius norm:

\beq{eq:norm2_err}{
\nF{\hAl - \optA}  \leq \frac{1 + \gamma}{\gamma \kappa^2} \lambda_T \sqrt{ds}
}

\item for the $\ell^q$ norm, with $q \in [1,2]$ and $\tau>0$:

\beq{eq:normq_err}{
|\hAl - \optA|_q \leq (1 + \tau + \gamma)^{4/q - 2} (1 + \gamma)^{2-2/q} (2\tau)^{1 - 2/q} \gamma^{-1} \kappa^{- 2} \lambda_T (d s)^{1/q}.
}
\end{enumerate}
\end{corollary}
All these inequalities are consequences of~Equation~\eqref{eq:norm1_L2_err}, and are proved in Section~\ref{ss:master_theorem}.
The next Theorem is a minimax lower bound over row-sparse matrices, for the considered model.


\begin{theo}
\label{th:sparse_lower_bound}
For some constants $c>0$ and $c'>0$, we have:
\beq{eq:th_lower_bound}{
\inf_{\hat{\mA}} \sup_{\mA \in \Gamma_s} \Espq{\mA}{ \nF{\widehat \mA - \mA}^2} \geq \frac{c'ds \log (cd/s)}{T},
}
where $\Gamma_s$ is the set of row-$s$-sparse matrices and the infimum is taken over all possible estimators.
\end{theo}

The proof of Theorem \ref{th:sparse_lower_bound} above is in Section \ref{ss:lower_bound}. It uses the approach from \cite{tsybakov:2008}, where we construct a set of matrices that are separated enough in Frobenius norm but close enough in terms of the resulting probability densities. 
For this, we need a set of row-$s$-sparse matrices that are invertible, that we create using regular graph adjacency matrices. 
The complexity of this set is controlled thanks to precise combinatorial results, such as the ones from~\cite{mckay:1991}.


Finally, we present an interesting by-product of the proof of 
Theorem~\ref{th:1norm_empirical_bound}. 
Theorem~\ref{th:RE} below expresses that a Restricted Eigenvalue condition, see~\cite{bickel2009}, is, quite surprisingly, satisfied in the case of the Ornstein-Uhlenbeck drift estimation, while it is well-known to be a mandatory assumption for the linear regression model, see~\cite{zhang2014lower}, when one wants to prove optimal convergence rates for polynomial-time sparsity inducing algorithms, such as $\ell_1$ penalization.

\begin{theo}
\label{th:RE}
Assume (H1) -- (H4). Set $s \leq d$ and $c_0 >0$. Define $C(s,c_0) :=  \{ \vect{u} \in \R^d: \nO{\vect{u}} \leq (1 + c_0) \nO{\vect{u}_{\mid \mathcal{I}_s(\vect{u})}} \}$ where $\mathcal{I}_s(u)$ stands for the set of indices of the $s$ largest entries of $|u|$. 
Consider $\epsilon_0 \in (0,1)$ and $T_0$ given in Theorem~\ref{th:1norm_empirical_bound}. 
Then, for any $T \geq T_0$, we have
\beq{eq:RE}{
\Prob{\inf_{\vect{u} \in C(s,c_0)} \frac{\nL{\vect{u}^\top X}}{\nT{\vect{u}}} \geq \kappa } \geq 1 - \frac{\epsilon_0}{2}.
}
\end{theo}

The proof of Theorem~\ref{th:RE} is given in Section~\ref{ss:RE} and uses explicitly Assumption~(H4), which is proved in Theorem~\ref{th:H4}, see Section~\ref{ss:deviation_hypothesis}, for a symmetric $\optA$. 
We can interpret it equivalently as a lower bound on $\tr(\mA \II \mA^\top)$ (see Lemma~\ref{l:L2_quadraticform} from Section~\ref{ss:RE}), hence as a RE property for $\II$ over row-$s$-sparse matrices $\mA$. 
Observe that the values of $\kappa$ and $\epsilon_0$ are independent on $s$ and $c_0$ and the validity of Equation \eqref{eq:RE} depends on $s,c_0$ solely through the condition $T \geq T_0(s,c_0)$. In other words, the validity of a Restricted Eigenvalue property in our model is achieved as long as $T$ is large enough.

\section{Asymptotic oracle properties for Adaptive Lasso}
\label{s:asymptotic}

The MLE is asymptotically optimal, as observed with the asymptotic normality property from Equation~\eqref{eq:asymptotic_normality}. 
In this Section we derive similar properties for the $\ell^1$-penalized estimator.
Furthermore, another desirable property from a sparsity-inducing estimator is consistency in variable selection~\cite{buhlmann2011statistics}. 
We define it by the property that the support of a $\supp(\widehat \mA)$ converges to the support of the true parameter $\supp(\optA)$.
It is known in the context of Gaussian linear regression that the Lasso cannot satisfy both properties with the same parameter
 $\lambda$, see~\cite{zou:06} while the Adaptive Lasso does. 
The Adaptive Lasso in our context is defined as 
\beq{def:adalasso}{
\hAla = \argmin_{\mA  \in \R^{d \times d}} \cL_T(\mA) + \lambda \nO{\mA \circ |\hA|^{-\gamma}},
}
for fixed positive parameters $\lambda$ and $\gamma$, where $\circ$ stands for the Hadamard product, and $|\hA|^{-\gamma}$ stands for the matrix obtained by computing entrywise the absolute value, and exponentiation by $-\gamma$ of the MLE estimator~\eqref{eq:mle_estimator}.
The idea of the Adaptive Lasso, involving a penalization level proportional to the entries of $|\hA|^{-\gamma}$ (any $\sqrt T$-consistent estimator can be used theoretically), is to penalize more the entries expected to be actually zeros (trusting the MLE) and to penalize less those expected to be non-zero. 
Note that the MLE entries are non-zero almost surely. 

Note that many quantities, such as $\lambda$ and estimators $\hA$ and $\hAla$, implicitly depend on $T$, and that we consider in this section asymptotics $T \rightarrow +\infty$.


\begin{theo}
\label{th:adalasso_oracle}
Assume (H1) -- (H2). Fix $\gamma >0$ and assume that $\lambda$ verifies $\lambda T^{1/2} \to 0$ and $\lambda T^{(\gamma+1)/2} \to +\infty$ when $T \to +\infty$. Then, we have the following properties.

\begin{enumerate}
\item Consistency of the variable selection: as $T \to +\infty$, we have
\beq{eq:oracle_selection}{
\Prob{\supp (\hAla) = \supp(\optcA)} \rightarrow 1.
}
\item Asymptotic normality: as $T \to +\infty$, we have
\beq{eq:oracle_normality}{
\sqrt{T} \left( (\hAla)_{\mid \optcA} - (\optA)_{\mid \optcA} \right) \xrightarrow{\cL} \cNN{0,\left(( \Vinf \otimes \id )_{\mid \optcA \times \optcA}\right)^{-1}},
}
where $\optcA = \supp(\optA)$ is the support of the parameter $\optA$.
\end{enumerate}
\end{theo}


The proof of Theorem~\ref{th:adalasso_oracle} is in Section \ref{ss:asymptotic}.
It expresses two crucial asymptotic behaviors of the Adaptive Lasso for the Ornstein-Uhlenbeck drift estimation. 
The first point shows that Adaptive Lasso can be reasonably used for 
support recovery of the drift parameter, whenever $T$ is large enough.
The second point proves that the Adaptive Lasso shares the property of asymptotic efficiency with the MLE, over the support of the true parameter.

\section{Numerical results}
\label{s:numerical}

This Section proposes numerical experiments, both on simulated and real datasets, that confirm our theoretical findings.
We start in Figure~\ref{f:comparison} with an illustration of estimation results using MLE, Lasso and Adaptive Lasso, where 
the advantage of penalized methods can be seen at a first glance.
The penalization level $\lambda$ of all estimators are tuned using a cross-validation procedure described in Section~\ref{ss:lambda_asymptotics} below.
\begin{figure}[htbp]
\centering
\includegraphics[trim={6.6cm 1cm 2cm 0.5cm},clip,width=.9\textwidth]{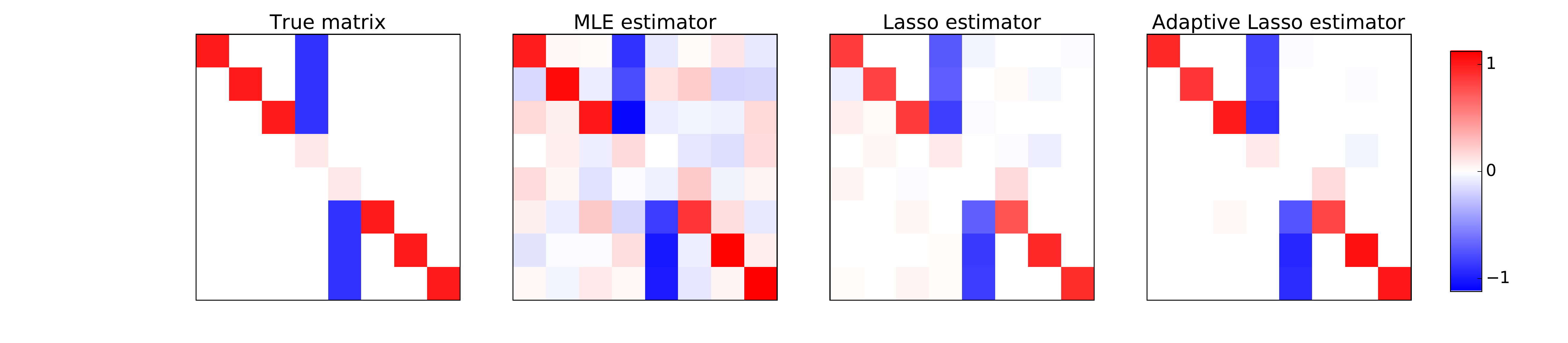}
\caption{MLE, Lasso and Adaptive Lasso estimates compared to the ground truth matrix. The Lasso shows a significant improvement over the MLE, and the Adaptive Lasso further improves on the Lasso, especially in terms of support recovery. All penalization parameters are obtained through cross-validation.}
\label{f:comparison}
\end{figure}

In the next sections we verify this observation using repeated experiments in different settings, in order to see the impact on estimation performance of the observation length $T$ and the dimension of the process $d$ (Section~\ref{ss:T_impact}).
The support recovery ability of Lasso and Adaptive Lasso are illustrated in Section~\ref{ss:support_recovery} and a brief analysis of the issue of trajectory discretization is discussed in Section~\ref{ss:ts_impact}. 
An application to real-world financial data is proposed in Section~\ref{ss:data}.
In all our experiments, we use a time-step equal to $10^{-2}$ for the discretization of the continuous trajectories, 
see Section~\ref{ss:ts_impact} for details.

\subsection{Cross-validation for selection of $\lambda$}
\label{ss:lambda_asymptotics}

The Lasso and the Adaptive Lasso use a parameter $\lambda$ that must be fixed by the user.
Our theoretical results suggest a value for $\lambda$ for the Lasso, see Equation~\eqref{def:lambda}. 
However, theoretical penalization parameters are known to be very pessimistic, in the sense that they are typically too large in most situations, see for instance~\cite{buhlmann2011statistics}. 
We propose instead to tune $\lambda$ through cross-validation.

In our setting, we implement cross-validation by using the first $80\%$ of the trajectory as the training set and the remaining $20\%$ as the validation set, in the following way.
\begin{align}
\hAl_\lambda &= \argmin_{\mA \in \R^{d \times d}} \cL_{.8 T}(\mA) + \lambda \nO{\mA}, \\
\hat \lambda &= \argmin_{\lambda \geq 0} \cL_{[.8 T,T]}(\hAl_\lambda),
\label{eq:cv_lambda}
\end{align}
where $\cL_{[.8 T,T]}$ is the negative log-likelihood constructed from the interval $[.8 T, T]$. 
The cross-validated Lasso is then $\ha_{\hat \lambda}$, and we define similarly the cross-validated adaptive Lasso.
In the next sections, referring to Lasso and Adaptive Lasso will always correspond to the Lasso and Adaptive Lasso with cross-validated $\lambda$.
Note that the selection of $\lambda$ is performed in a grid on a logarithmic scale between $10^{-2}$ and $10^3$, in all our experiments.
 


\subsection{Influence of the observation length $T$ and of the dimension $d$}
\label{ss:T_impact}

In Figures~\ref{f:d_impact_s} and~\ref{f:T_impact}, we plot the average Frobenius norm estimation error of MLE, Lasso and Adaptive Lasso, respectively as a function of the dimension~$d$, and as a function of the sample size $T$.
The bold lines and the shaded areas correspond respectively to the means and standard deviations of the errors obtained over 100 independent simulated trajectories.
The ground truth parameter $\optA$ is chosen as a random matrix with sparsity equal to $0.2d$, with non-zero entries equal to $\pm 1$.
Such a matrix is displayed for $d = 80$ on the left-hand size of Figure~\ref{f:d_impact_s}.
Note that all $y$-axis are on a logarithmic scale.

In Figure~\ref{f:d_impact_s}, we observe the deterioration of the estimation error with an increasing dimension~$d$. 
We observe also that penalized procedures perform much better than the MLE, 
but that slopes are very close: this comes from the fact that the row sparsity is fixed to $0.2 \times d$ leading to a $0.2 \times d^2$ overall sparsity, which is not much smaller than the dense case $d^2$ for the range of values of $d$ considered in the experiment.

In Figure~\ref{f:T_impact}, we observe the improvement of the estimation error with an increasing sample size $T$.
We observe that the curves are consistent with a common convergence rate of order $\sqrt{T}$, but that penalized estimates constantly outperform the MLE.

\begin{figure}[htbp]
\centering
\begin{subfigure}{0.26\textwidth}
\centering
\includegraphics[trim={1cm 0.3cm 1cm 0.3cm},clip,height=.96\textwidth]{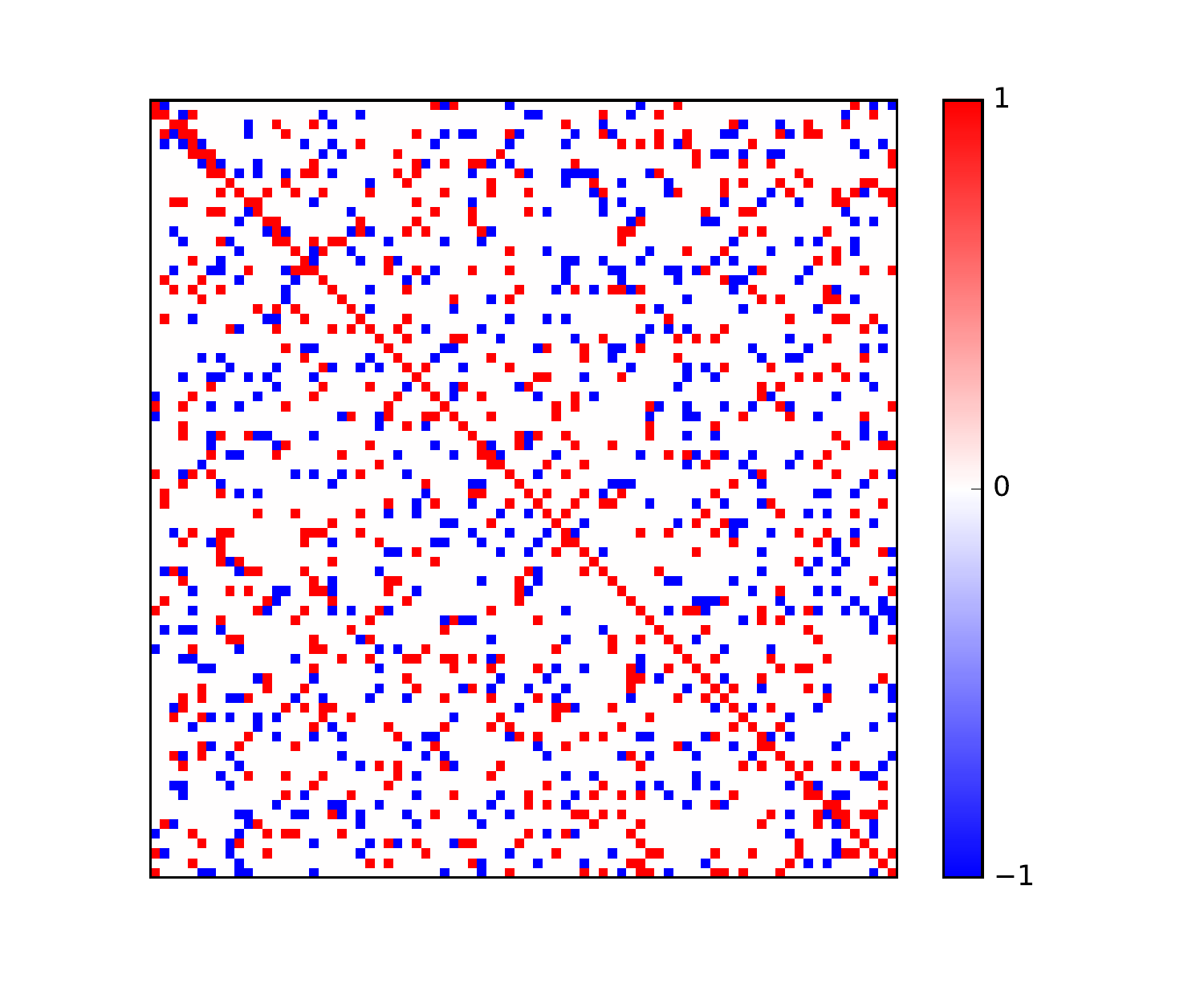}
\caption{Ground truth $\optA$}
\label{f:BWR_A80}
\end{subfigure}
\hfill{}
\begin{subfigure}{0.3\textwidth}
\centering
\includegraphics[trim={1cm 0.3cm 1cm 0.3cm},clip,height=.8\textwidth]{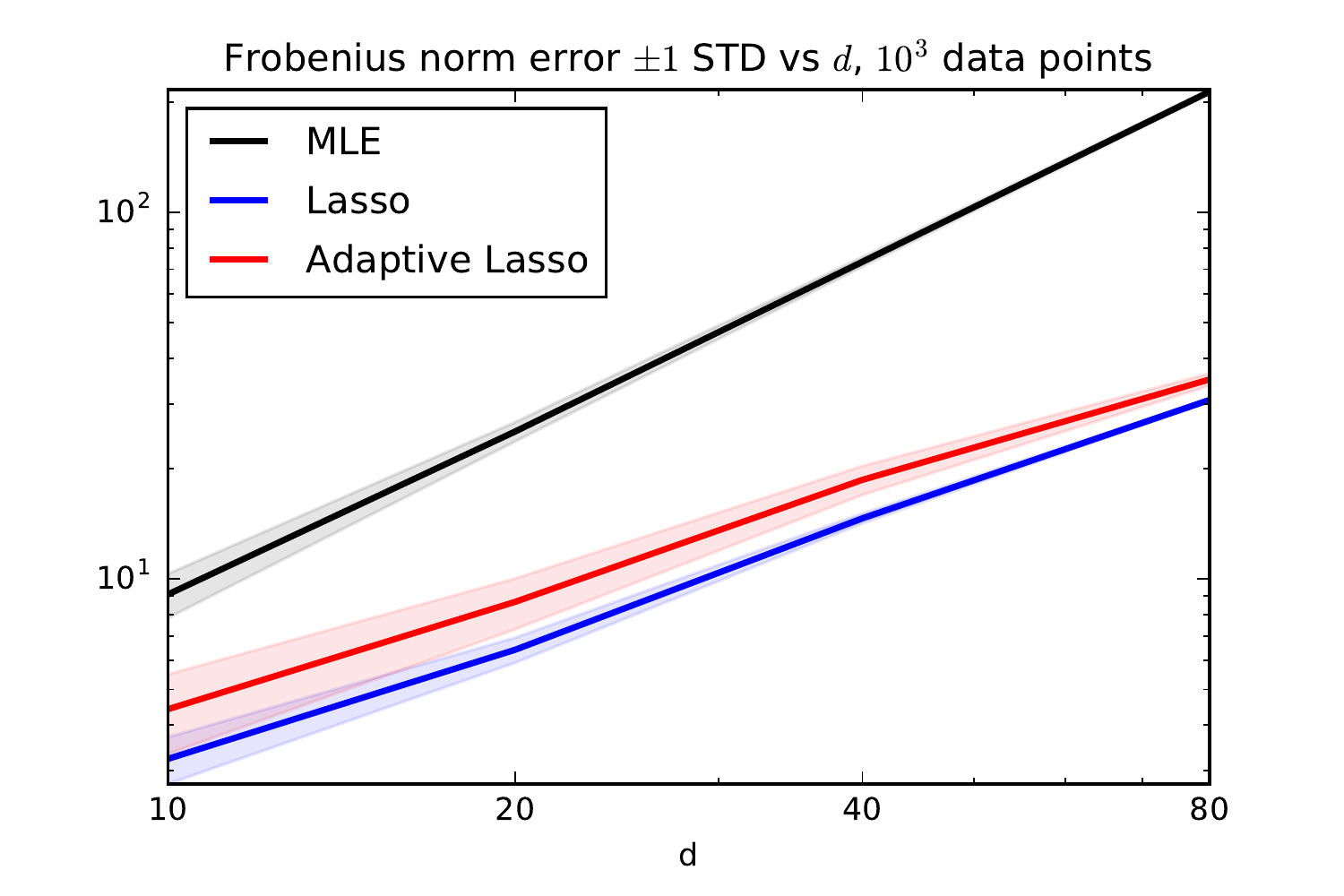}
\caption{Errors with $T=10$}
\label{f:matrix}
\end{subfigure}
\hfill{}
\begin{subfigure}{0.3\textwidth}
\centering
\includegraphics[trim={1cm 0.3cm 1cm 0.3cm},clip,height=.8\textwidth]{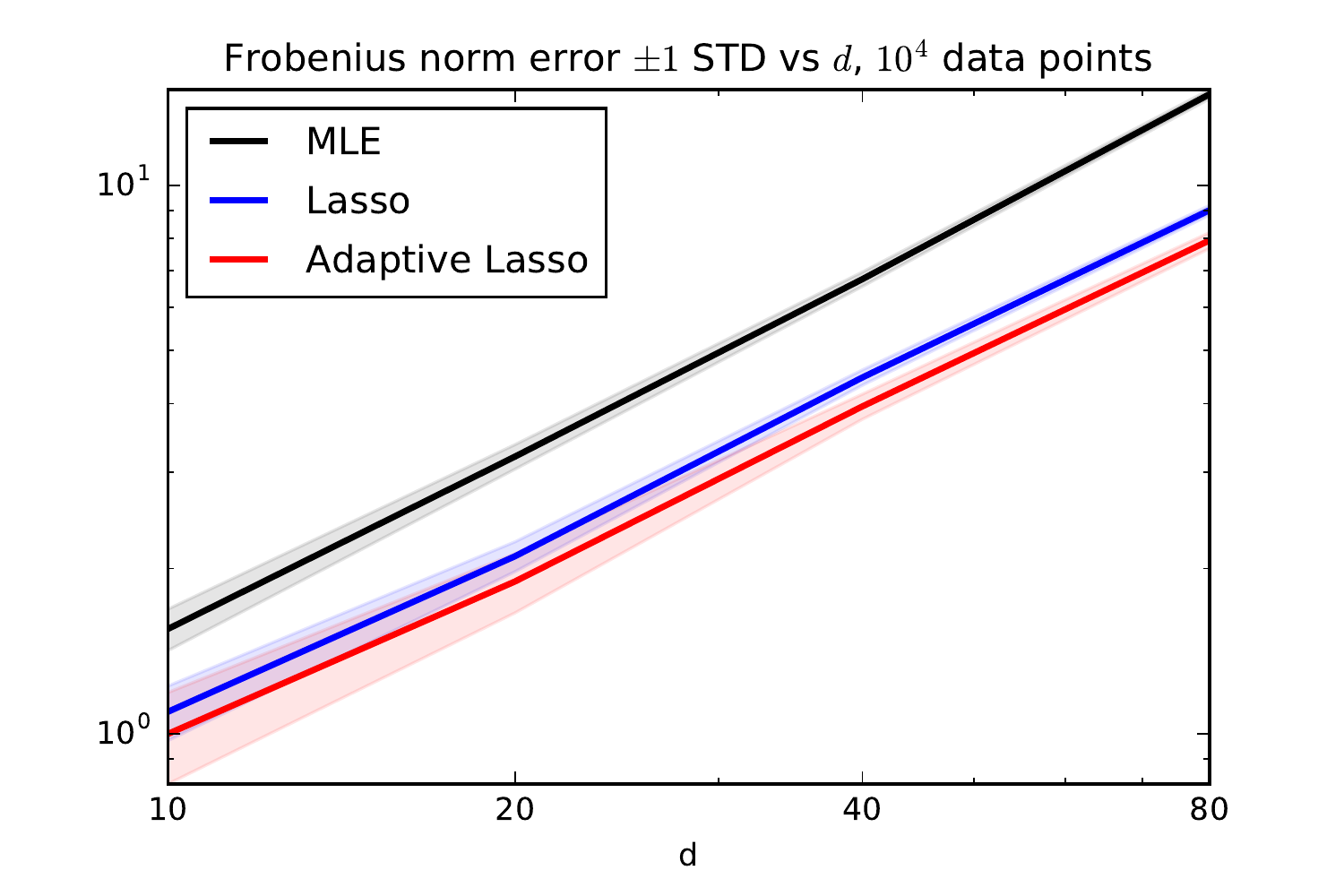}
\caption{Errors with $T=100$}
\label{f:matrix}
\end{subfigure}
\caption{(a): Example of a simulated ground truth matrix, with row sparsity equal to $0.2 d$; (b): estimation errors measured by the Frobenius norm for Lasso, Adaptive Lasso and MLE, as a function of $d$, with a sample
 size $T = 10$; (c): same as (b) with $T = 100$.
Bold lines and shaded areas correspond respectively to the means and standard deviations of the error obtained over 100 independent simulated trajectories. 
}
\label{f:d_impact_s}
\end{figure}

\begin{figure}[htbp]
\centering
\begin{subfigure}{0.45\textwidth}
\centering
\includegraphics[width=\textwidth]{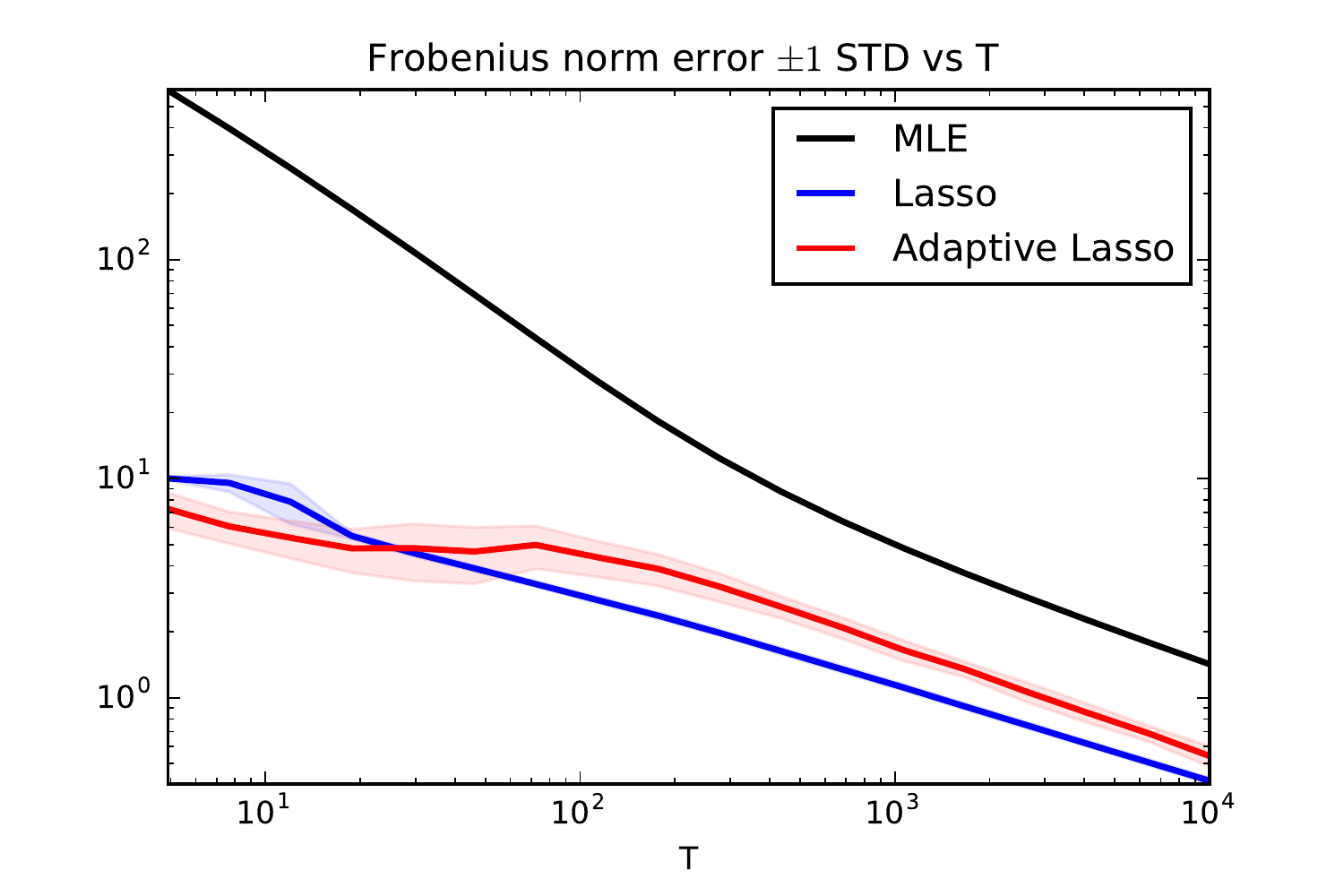}
\caption{Estimation error for $d=10$}
\label{f:l1_3}
\end{subfigure}
\begin{subfigure}{0.45\textwidth}
\centering
\includegraphics[width=\textwidth]{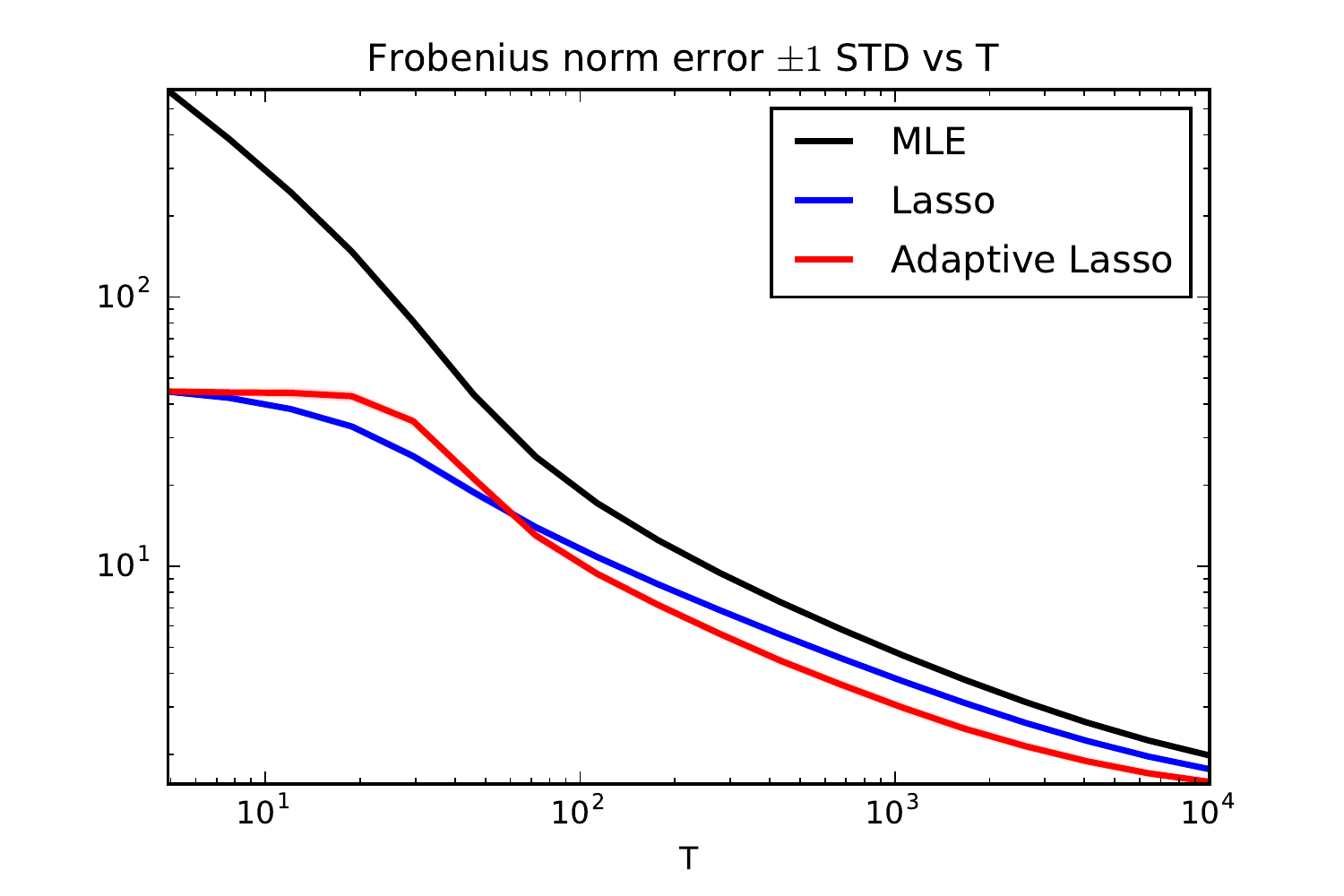}
\caption{Estimation error for $d = 100$}
\label{f:l1_4}
\end{subfigure}
\caption{Estimation errors measured by the Frobenius norm for Lasso, Adaptive Lasso and MLE, as a function of $T$ for (a) $d=10$ and (b) $d=100$.
Bold lines and shaded areas correspond respectively to the means and standard deviations of the error obtained over 100 independent simulated trajectories.}
\label{f:T_impact}
\end{figure}

\subsection{Support recovery}
\label{ss:support_recovery}

Penalization methods such as Lasso and Adaptive Lasso can be used for variable selection, because of their sparsity-inducing property. 
Indeed, we proved in Theorem~\ref{th:adalasso_oracle} from Section~\ref{s:asymptotic} that the Adaptive Lasso is consistent for variable selection of the drift parameter $\optA$.
In Figure~\ref{f:fscores}, we consider the estimation problem of a $80 \times 80$ matrix $\optA$ with sparsity $0.1 \times d$, and with a 
sample size $T=100$.
We simulate 100 trajectories, and compute the $F_1$-score obtained for support selection achieved by the MLE, Lasso and Adaptive Lasso.
Figure~\ref{f:fscores} then displays the box-plots of these $F_1$-scores.
The $F_1$-score obtained by each estimator is computed as follows:
first, we binarize the entries of the estimators and of the ground-truth matrix: zero entries are kept equal to zero, while non-zero entries are replaced by ones.
Then, we count the true positives, false positives and false negatives in order to compute the precision and recall values.

The MLE does not lead to a sparse solution, so that its $F_1$-score is constant and is computed from the average row-sparsity $s$ using the fromula $2s/(1+s)$, which is around $0.2$ in our case as observed in Figure~\ref{f:fscores}.
Indeed, the MLE always classifies all entries as non-zero, hence the corresponding true positive, false positive and false negative values are always equal respectively to $sd$, $d^2$ and $0$.
The strong improvement of the Adaptive Lasso over Lasso is clearly illustrated on this example: its $F_1$-score is almost equal to $1$, while the Lasso achieves an $F_1$-score slightly below $0.5$.

\begin{figure}[htbp]
\centering
\includegraphics[width=.5\textwidth]{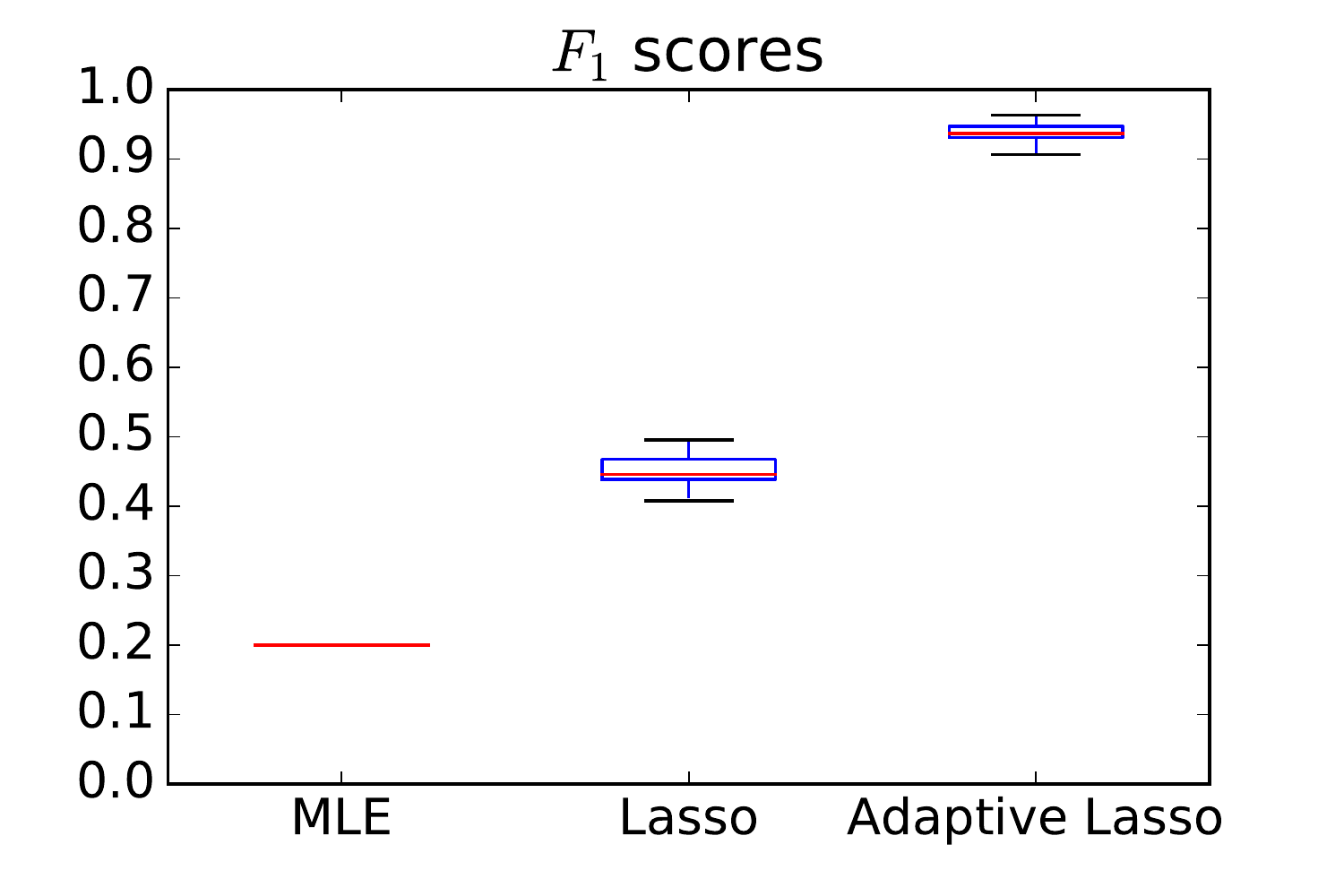}
\caption{Accuracy scores for different estimation methods, using a $80 \times 80$ matrix $\optA$ and an observation length $T=100$. We classify as positive detection all non-zero entries of the estimators. The plot illustrates a clear advantage of Adaptive Lasso over Lasso in terms of support recovery. 
The MLE accuracy is provided here for convenience, as a lower-bound for any procedure, that corresponds to the sparsity of $\optA$.}
\label{f:fscores}
\end{figure}

\subsection{Influence of the time-step}
\label{ss:ts_impact}

The theoretical results proposed in this paper assumes a continuous observation of the trajectory of the data.
However, in practice, any simulation or real-data analysis will have to use some discretization method. 
In our simulations we use a constant time-step $\delta t = 10^{-2}$. 
This value has been decided in hindsight, after a study of the impact of the time-step on the quality of the estimators. 
This study is illustrated in Figure~\ref{f:l1_ts}, where we display box-plots of the estimation errors obtained with varying discretization time steps $\delta t \in \{1,10^{-1},10^{-2},10^{-3},10^{-4}\}$.
We observe in~Figure~\ref{f:l1_ts} that results improved with a decreasing $\delta t$, which is to be expected, since a smaller time step means more data points, but we observed  no improvement for $\delta t \leq 10^{-2}$.

\begin{figure}[htbp]
\centering

\begin{subfigure}[b]{0.45\textwidth}
\label{f:l1_ts_adalasso_errors_ran50}
\centering
\includegraphics[width=\textwidth]{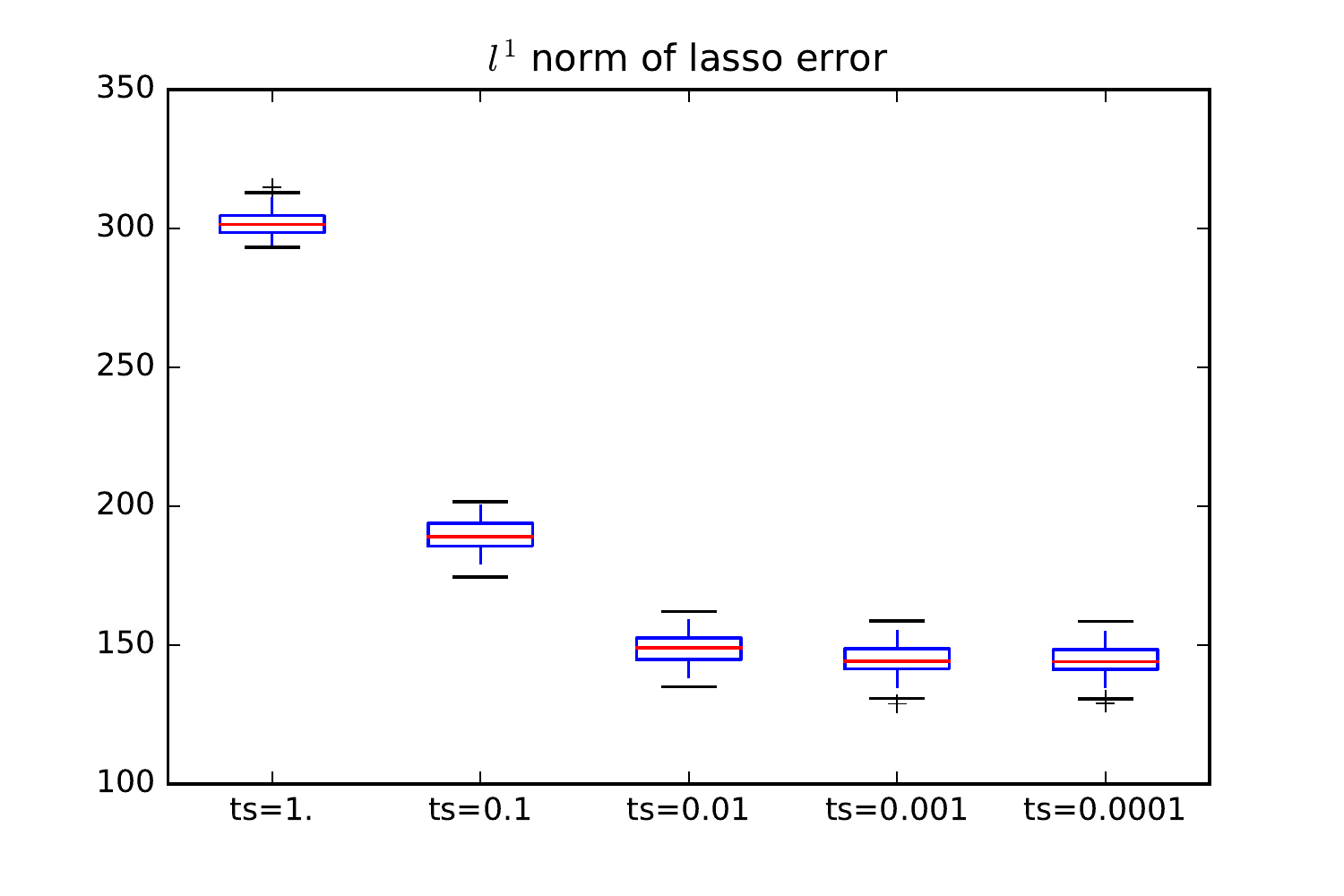}
\caption{$\ell^1$ error for the Lasso estimator.}
\end{subfigure}
~
\begin{subfigure}[b]{0.45\textwidth}
\centering
\label{f:F_ts_adalasso_errors_ran50}
\includegraphics[width=\textwidth]{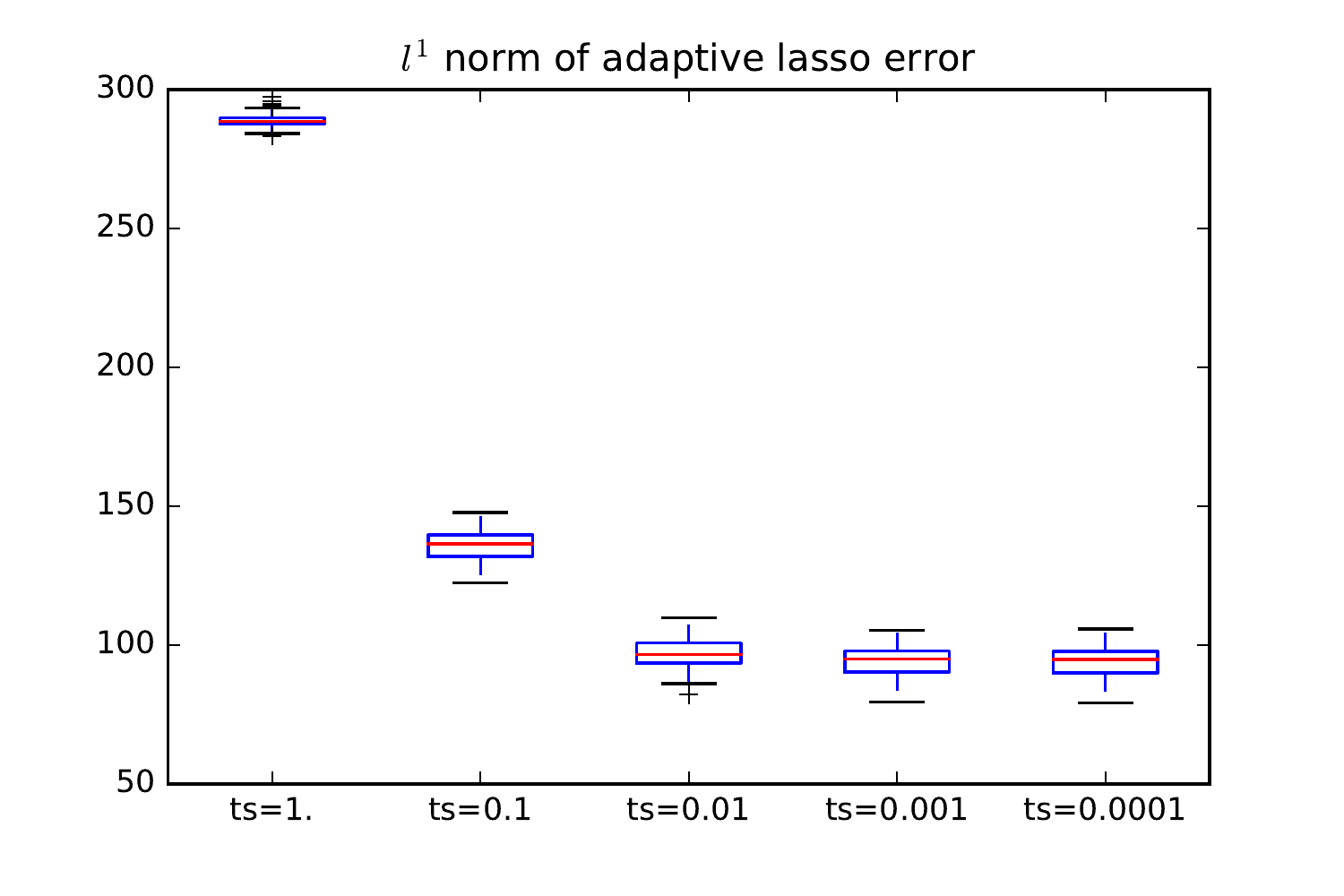}
\caption{$\ell^1$ error for the Adaptive Lasso estimator.}
\end{subfigure}

\caption{Estimation errors for Lasso and Adaptive Lasso as a function of the discretization step $\delta t$. Box-plots are computed from the estimation errors obtained from $100$ simulations of the same process, with a decreasing time step for discretization.}
\label{f:l1_ts}
\end{figure}

\subsection{Application to financial data}
\label{ss:data}

The Ornstein-Uhlenbeck model is a popular method in finance to model mean-reverting processes, for instance for pairs trading~\cite{hull:2009,carmona:2013,fouque:2013}.
In a typical setting, one chooses two related financial assets (for example the stocks of two companies in the same sector). Upon verification in the data, it is assumed that some linear combination of the stocks reverts to some "normal" value, often chosen as $0$. 
This combination, denoted $X$, can be modeled by a one-dimensional Ornstein-Uhlenbeck process.

However, this method does not address two issues. First, one needs a separate method to find the relevant pairs. Second, instead of pairs, one could be interested in more general linear combinations or in situations where the evolution of one price impacts another, when the pair does not fit a mean-reverting process. The multi-dimensional Ornstein-Uhlenbeck process is a way to address both problems, as it allows to involve an unrestricted number of assets. Moreover, thanks to the sparsity-inducing penalization considered in this paper, a sparse estimator might help in finding relevant combinations.

To illustrate this, we take daily close data of SP500 stocks, for companies in the financial and IT sectors with long enough history in the SP500 index. Our choice of sectors is arbitrary and is motivated by simplicity. We take the log-returns, then compute the exponential moving average (EMA), which will be the data we want to model using an Ornstein-Uhlenbeck process. 
By design, the EMA has a mean-reverting property and hence is a good candidate for fitting an Ornstein-Uhlenbeck process. 
We denote that process $R$ and assume the model:
\beq{eq:log_returns}{
\dd R_t = - \mA (R_t - m) \dt + \mSigma \dW_t
}
where $\mSigma$ is typically not the identity because of high correlations between certain stocks. We estimate $m$ and $\mSigma$ using the mean and the squared variations. Because of $\mSigma$, in order to estimate $\mA$, we need to maximize a slightly modified log-likelihood which takes $\mSigma$ into account, which is done easily using a proximal gradient descent algorithm for instance. 
The resulting MLE and cross-validated Adaptive Lasso give the matrices in Figure~\ref{f:sp500_a_estimation}. The heavy diagonals are explained by the fact that the data is an exponential moving average. 
However, the non-zero values away from the diagonal are non-trivial, since they can't be explained by covariances, that were already captured by the estimation of $\mSigma$. 
\begin{figure}
\centering
\begin{subfigure}[b]{0.3\textwidth}
\label{f:sp500_mle}
\centering
\includegraphics[trim={1cm 0.3cm 1cm 0.3cm},clip,width=\textwidth]{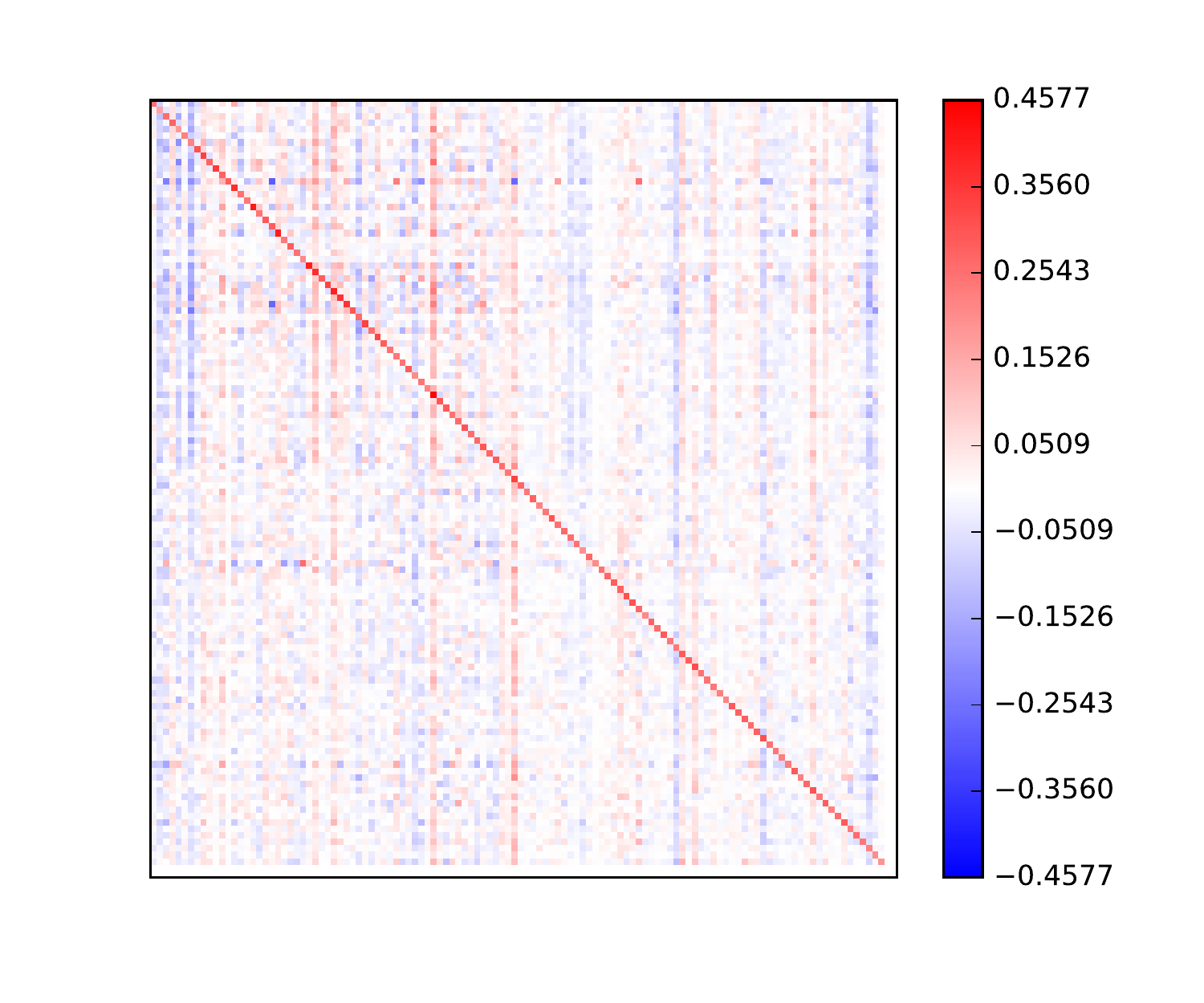}
\caption{Estimation using MLE}
\end{subfigure}
\hfill
\begin{subfigure}[b]{0.3\textwidth}
\centering
\label{f:sp500_adalasso}
\includegraphics[trim={1cm 0.3cm 1cm 0.3cm},clip,width=\textwidth]{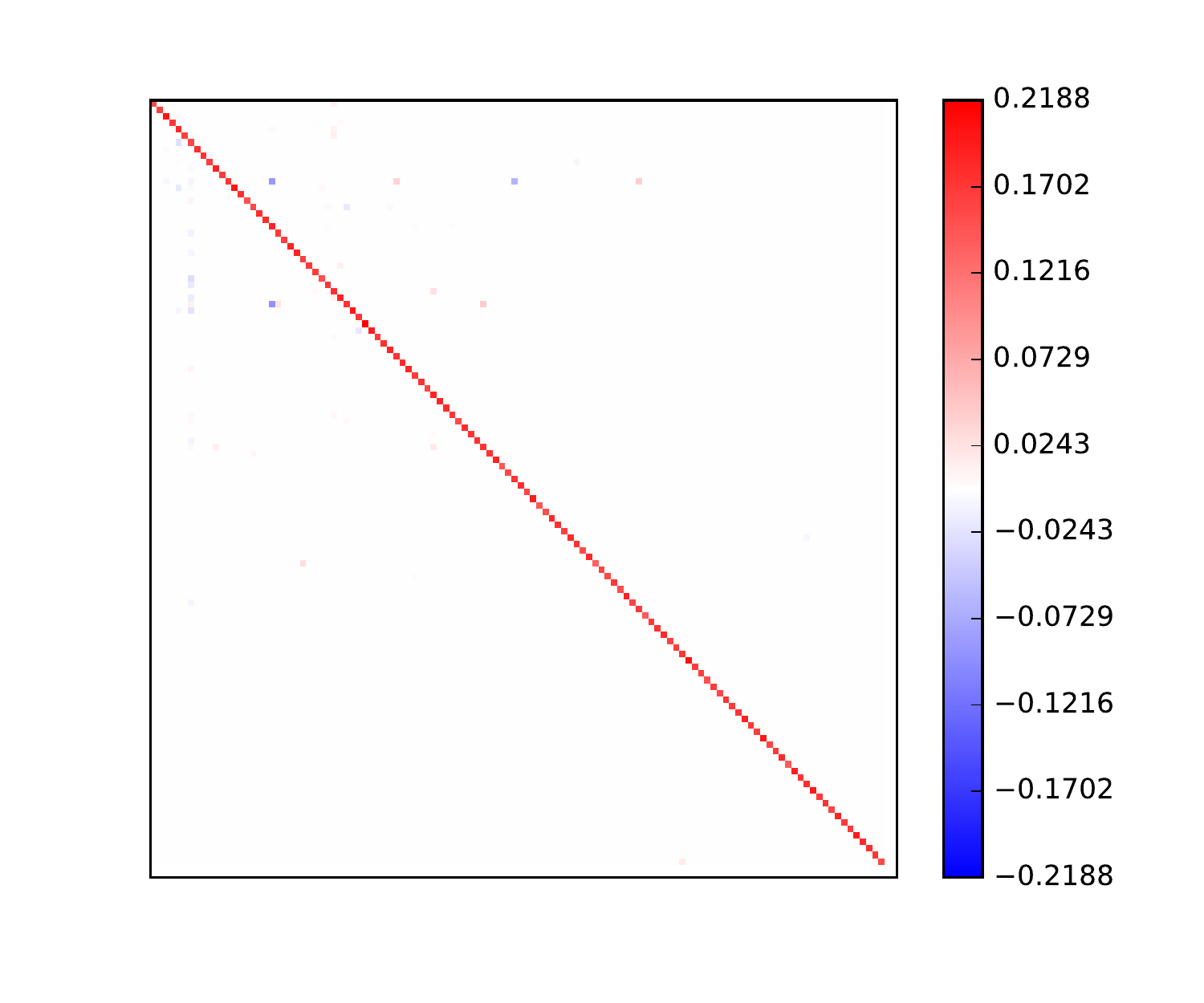}
\caption{Estimation using Ad. Lasso}
\end{subfigure}
\hfill
\begin{subfigure}[b]{0.3\textwidth}
\centering
\label{f:sp500_adalasso_zoom}
\includegraphics[trim={1cm 0.3cm 1cm 0.3cm},clip,width=\textwidth]{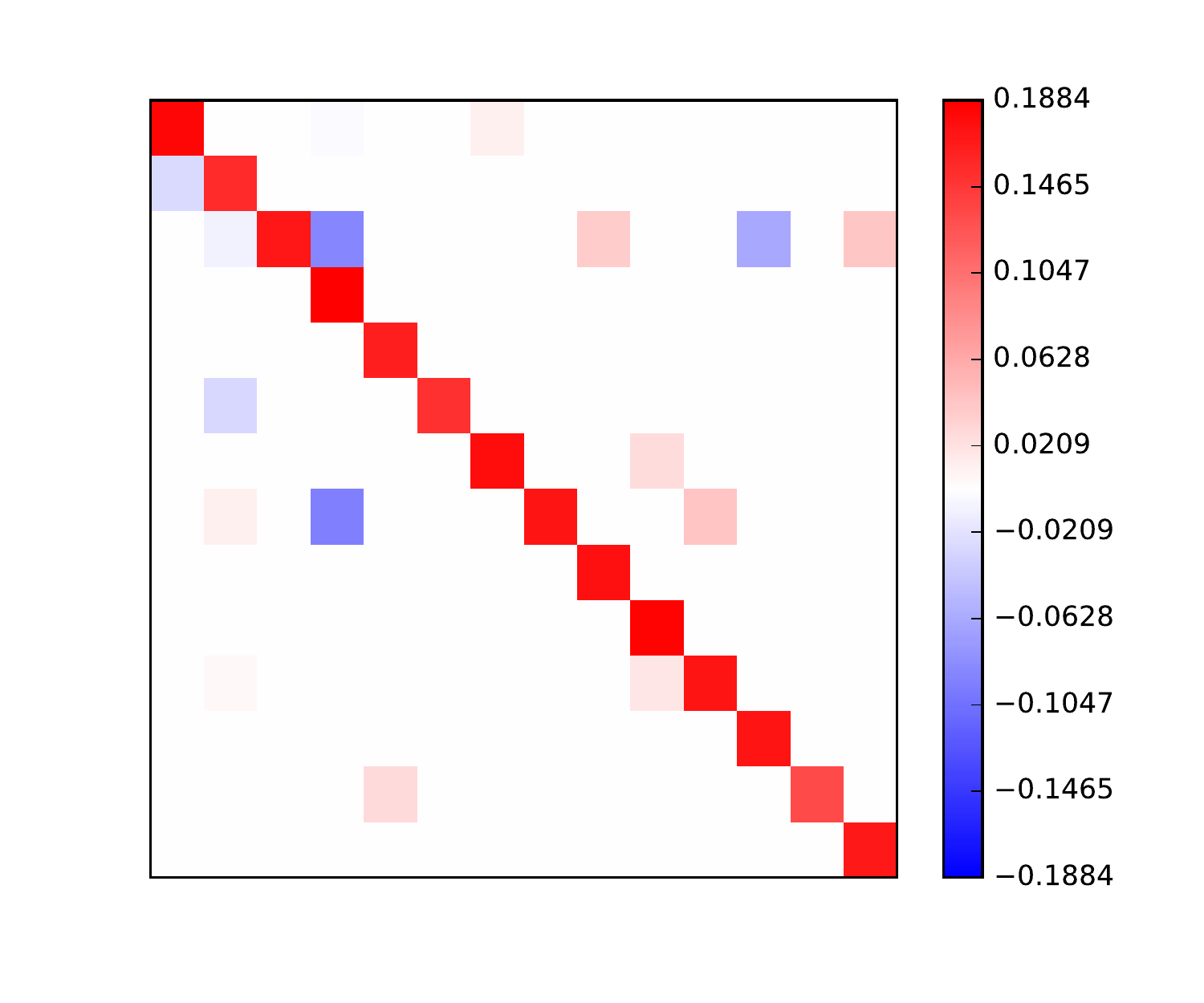}
\caption{Zoom of Ad. Lasso}
\end{subfigure}

\caption{(a) MLE estimator; (b) Adaptive Lasso estimator; (c) Zoom of (b) for stocks with the highest non-diagonal values; all for the estimation of $\mA$ in the model of Equation~\eqref{eq:log_returns}. The diagonal values are expected from the design of the EMA. The MLE gives a very noisy estimate, while the Adaptive Lasso is highly sparse.}
\label{f:sp500_a_estimation}
\end{figure}

The sparse estimation selects the most significant stock prices that influence other stock prices. 
This can be an indication to find interesting stock pairs. 
The highest value, in absolute value, that we find outside of the diagonal is at tick-coordinates ('PRU','FITB'), and takes a value roughly equal to $-0.1$. 
This means in practice that given an above-average value for the exponential moving average of log-returns of FITB, the model predicts an increase of the exponential moving average of log-returns of PRU, all else being controlled: by controlling the correlation between the two stocks, we get the plot from Figure \ref{f:stocks}. 
In layman terms, recent above-average returns of FITB predict above-average returns of PRU.
As a disclaimer, we should point out that this study has been conducted with a very simple approach, and shouldn't be considered as trading advice.

\begin{figure}
\centering
\includegraphics[width=.5\textwidth]{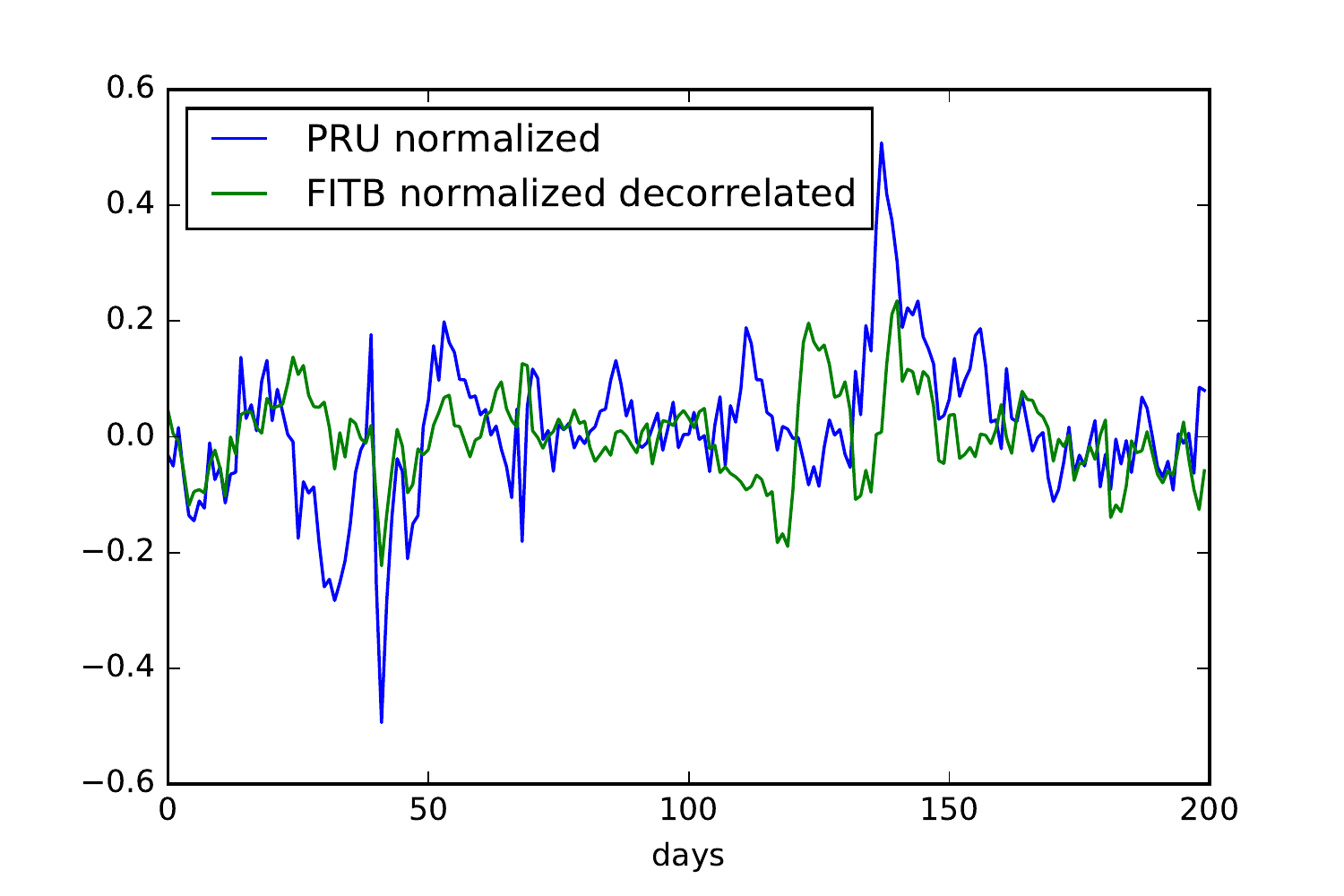}
\caption{Plot of normalized EMA for PRU and FITB, the latter subtracted a fraction of the PRU data for decorrelation.}
\label{f:stocks}
\end{figure}

\section{Conclusion}

This paper provides a complete theory for the estimation of the drift parameter of a Ornstein-Uhlenbeck processes under a row-sparsity assumption.
This is, to the best of our knowledge, the first paper to provide such results, either in a non-asymptotic or asymptotic framework, for Lasso and Adaptive Lasso.
This paper is therefore a first attempt towards the use of sparsity-inducing penalization, widely used in the context of generalized linear models, to high-dimensional diffusion processes.

A natural extension of our work consists in assuming a correlated Brownian noise, modeled by a non-diagonal parameter $\mSigma$ in front of $\dW_t$ in Equation~\eqref{def:OU}.
This parameter is exactly computable in the continuous observation setting.
However, in a high-dimensional setting, $\mSigma$ should be considered sparse as well, and one could therefore consider a joint estimation procedure for $\mA$ and $\mSigma$, with dedicated sparsity-inducing penalizations.
However, it turns out to be a much more difficult task, since the negative log-likelihood is not jointly convex with respect to $\mA$ and $\mSigma$.
Such a development is therefore way beyond the scope of the present paper, and might actually involve a very different approach than the one considered here.

Another natural extension is to consider matrices $\optA$ with non-positive spectra.
A very interesting property is zero eigenvalues, which leads to a reduced rank and hence to co-integrated processes.
A method to reduce rank is to penalize it, see~\cite{bunea:2011,bunea:2012} for application to Gaussian regression, or to use the so-called trace norm or nuclear norm penalization, which corresponds to a convex relaxation of the rank.


\section{Proofs}
\label{s:proofs}

In this Section, we provide proofs of the theorems and other statements from Sections \ref{s:nonasymptotic} and \ref{s:asymptotic}.

\subsection{Proof of Assumption (H4) in the reversible case}
\label{ss:deviation_hypothesis}

Theorem \ref{th:H4} below expresses that assumption (H4) is true when $\optA$ is symmetric. This condition is equivalent in our case to the reversibility of the process, see \cite{gobet:2016b}.

\begin{theo}
\label{th:H4}
Assume that $\optA$ is symmetric. Then there exists a non-decreasing, non-negative function $H$ such that for any vector $\vect{u}$, $\nT{\vect{u}} \leq 1$, we have:

\beq{eq:h4_app}{
\Prob{| \vect{u}^\top (\II - \Vinf) \vect{u} | \geq R} \leq 2 \exp (-T H(R)).
}
\end{theo}
Theorem \ref{th:H4} above follows from Lemmas \ref{l:geqR} and \ref{l:leqminusR} below, after taking $H(R) = H_1(R) \wedge H_2(R)$. The proof of the Lemmas is based on Theorem \ref{th:log_sobolev_bound} below, which shows a deviation inequality for the integral of a functional of an ergodic process from its long-time limit.



\begin{theo}[\cite{cattiaux:2007}, Theorem 2.1]
\label{th:log_sobolev_bound}
Let $L$ be the infinitesimal generator associated to an ergodic diffusion $X$ with stationary distribution $\mu$. If $\mu$ satisfies the $\log$-Sobolev inequality:
\beq{eq:log-Sobolev}{
c \int f^2 \log f^2 \dd\mu \leq - < Lf, f>_\mu
}
for some $c > 0$ and for all functions $f$ in the domain of definition of $L$ such that $\int_{\R^d} f^2 \dd \mu = 1$, then for all $Q \in \dL^1(\mu)$ and $R >0$:
\beq{eq:deviation_bound}{
\Prob{\frac{1}{T} \int_0^T Q(X_t) \dt - \int Q \dd\mu \geq R} \leq \exp \left( -t H^* \left( R \right) \right)
}
where
\beq{eq:exponential_bound}{
H^*(R) := \sup_{0 \leq \rho <\rho_{max}} \left \{ \rho R - c \log \int \exp \left( \frac{\rho}{c} (Q - \int Q \dd \mu) \right) \dd \mu \right \} 
}
and $\rho_{max}$ is such that the integral above is finite for any $0 \leq \rho < \rho_{max}$.
\end{theo}

\begin{remark}
\label{r:log_sob}
When $\optA$ is symmetric, a simple integration by parts shows that $< Lf, f>_\mu := \int_{\R^d} Lf \cdot f \dd \mu = - \frac{1}{2} \int_{\R^d} \nF{\nabla f}^2 \dd \mu$ and hence that Equation \eqref{eq:log-Sobolev} holds due to the classical log-Sobolev inequality \cite{gross:1975}, with $c = 1/4$. 
\end{remark}

Observe that Equation \eqref{eq:deviation_bound} applies to a one-sided inequality. Therefore, in order to get Theorem \ref{th:H4}, we will have to work with two inequalities. We deal with the first one in Lemma \ref{l:geqR} below.

\begin{lemma}
\label{l:geqR}
Assume $\optA$ is symmetric. Then for any vectors $\vect{u}$ such that $\nT{\vect{u}} \leq 1$:

\[ \Prob{\vect{u}^\top (\II - \Vinf) \vect{u} > R} \leq \exp( -T H_1(R))\]

where $H_1(R) = \frac{1}{8} \left(  \frac{R}{\vect{u}^\top \Vinf \vect{u}} - \log \det \left(\id  + R \frac{ \Vinf \vect{u} \vect{u}^\top}{(\vect{u}^\top \Vinf \vect{u})^2 } \right) \right)$.

\end{lemma}

\begin{proof}
Observe first that it suffices to prove the Lemma for $\nT{\vect{u}} = 1$. In the following, $\vect{u} \in \R^d$ verifies that condition.
We apply Theorem \ref{th:log_sobolev_bound}, which applies with $c=1/4$ as explained in Remark \ref{r:log_sob}, to the function $Q(X) = \vect{u}^\top X X^\top \vect{u} = (\vect{u}^\top X)^2$. Then $\int Q \dd \mu = \vect{u}^\top \Vinf \vect{u}$. It remains to write explicitly $H^*(R)$. We have $H^*(R) = \sup_{0 \leq \rho < \rho_{max}} \rho (R + \vect{u}^\top \Vinf \vect{u}) - \frac{1}{4} \log I_\rho$ where:

\begin{align}
I_\rho &:= \int \exp \left( 4\rho \vect{u}^\top X X^\top \vect{u} \right) \dd \mu \\
&= (2\pi)^{-d/2} (\det \Vinf)^{-1/2} \int \exp \left(- \frac{1}{2} X^\top \left( \Vinf^{-1} - 8 \rho \vect{u} \vect{u}^\top \right) X \right) \dX \\
&= (\det \Vinf)^{-1/2} (\det \mSigma_\rho)^{1/2} \label{eq:res_Ilambda}
\end{align}
and $\mSigma_\rho^{-1} := \Vinf^{-1} - 8 \rho \vect{u} \vect{u}^\top$. The product $\vect{u} \vect{u}^\top$ is a symmetric matrix of rank $1$ and its only non-zero eigenvalue is $1$, as $\vect{u} \vect{u}^\top \vect{u} = \vect{u}$. 

The integral $I_\rho$ is defined and Equation \eqref{eq:res_Ilambda} is valid if and only if $\mSigma_\rho^{-1}$ is indeed a matrix with positive spectrum. We have 
\[\min \Sp(\mSigma_{\rho}^{-1}) \geq \min \Sp(\Vinf^{-1}) - 8 \rho \max \Sp(\vect{u} \vect{u}^\top) = (\max \Sp(\Vinf))^{-1} - 8\rho.\]
Hence we choose $\rho_{max} := \frac{1}{8}(\max \Sp(\Vinf))^{-1}$ and $I_\rho$ is well defined for $\rho < \rho_{max}$. Note also for later that $\nOp{8 \rho \Vinf \vect{u} \vect{u}^\top} \leq 8\rho \max \Sp(\Vinf) < 1$.

To find $H^*(R)$, we differentiate the argument of the supremum. For this, we need $\frac{\dd \det \mSigma_\rho}{\dd \rho}$. We have $\frac{\dd \mSigma_\rho}{\dd \rho} = - \mSigma_\rho \frac{\dd \mSigma_\rho^{-1}}{\dd \rho}\mSigma_\rho = 8 \mSigma_\rho \vect{u} \vect{u}^\top\mSigma_\rho$. Hence 
\[\frac{\dd \det \mSigma_\rho}{\dd \rho} = \tr \left(\adj (\mSigma_\rho) \frac{\dd \mSigma_\rho}{\dd \rho} \right) = 8 \tr (\adj (\mSigma_\rho ) \mSigma_\rho \vect{u} \vect{u}^\top\mSigma_\rho ) = 8 (\det \mSigma_\rho) \vect{u}^\top \mSigma_\rho \vect{u}.\]
Therefore, to find $H^*(R)$, we solve in $\rho$ the equation $ R + \vect{u}^\top \Vinf \vect{u} - \vect{u}^\top \mSigma_\rho \vect{u} = 0$. We can actually compute $\vect{u}^\top \mSigma_\rho \vect{u}$ using a geometric series, recalling that $\nOp{8 \rho \Vinf \vect{u} \vect{u}^\top} < 1$:

\begin{align}
\mSigma_\rho &= \left( \id - 8\rho \Vinf \vect{u} \vect{u}^\top\right)^{-1} \Vinf = \sum_{k \geq 0} \left( 8 \rho \Vinf \vect{u} \vect{u}^\top \right)^k \Vinf \\
&= \Vinf + 8 \rho \Vinf \vect{u} \sum_{k \geq 0} \left( 8 \rho \vect{u}^\top \Vinf \vect{u} \right)^k \vect{u}^\top \Vinf \\
&= \Vinf + 8 \rho \frac{\Vinf \vect{u} \vect{u}^\top \Vinf }{1 - 8 \rho \vect{u}^\top \Vinf \vect{u} } \\
\vect{u}^\top \mSigma_\rho \vect{u} &= \frac{ \vect{u}^\top \Vinf \vect{u} }{1 - 8 \rho \vect{u}^\top \Vinf \vect{u} }.
\end{align}
We get $H^*(R)$ for $ \rho = \rho^* := \frac{1}{8} \frac{R}{\vect{u}^\top \Vinf \vect{u} (R + \vect{u}^\top \Vinf \vect{u})}$. Then:
\begin{align}
\Sigma_{\rho^*} 
&= \Vinf \left(\id  + R \frac{ \vect{u} \vect{u}^\top \Vinf}{(\vect{u}^\top \Vinf \vect{u})^2 } \right) \\
H^*(R) &= \rho^* (R + \vect{u}^\top \Vinf \vect{u}) + \frac{1}{8} \log (\det \Vinf) - \frac{1}{8} \log (\det \mSigma_{\rho^*}) \\
&= \frac{1}{8} \left(  \frac{R}{\vect{u}^\top \Vinf \vect{u}} - \log \det \left(\id  + R \frac{ \vect{u} \vect{u}^\top \Vinf}{(\vect{u}^\top \Vinf \vect{u})^2 } \right) \right). \qedhere
\end{align}
\end{proof}

We deal with the second inequality in Lemma \ref{l:leqminusR} below. 

\begin{lemma}
\label{l:leqminusR}
Assume $\optA$ is symmetric. Then for any vectors $\vect{u}$ such that $\nT{\vect{u}} \leq 1$:

\beq{eq:leqminusR}{
\Prob{ \vect{u}^\top (\II - \Vinf) \vect{u} \leq -R} \leq \exp( -T H_2(R))
}


where $H_2(R) = \begin{cases} -\frac{1}{8} \left( \frac{R}{\vect{u}^\top \Vinf \vect{u}} + \log \det \left(\id - R \frac{ \Vinf \vect{u} \vect{u}^\top}{(\vect{u}^\top \Vinf \vect{u})^2 } \right) \right) & \text{if } R < \vect{u}^\top \Vinf \vect{u}, \\ + \infty & \text{else.} \end{cases} $.

\end{lemma}

\begin{proof}
Observe first that if $R \geq \vect{u}^\top \Vinf \vect{u}$, the probability is zero, as $\II$ is a.s. a positive definite matrix. We assume henceforth that $R < \vect{u}^\top \Vinf \vect{u}$. The same reasoning as in the proof of Theorem \ref{l:geqR} to $Q(X) = - \vect{u}^\top X X^\top \vect{u}$ gives:
\begin{align}
H^*(R) &:= \sup_{\rho \geq 0} \rho R - \rho \vect{u}^\top \Vinf \vect{u} + \frac{1}{8} \log \det \Vinf - \frac{1}{8} \log \det \mSigma_\rho \\
\mSigma_\rho^{-1} &:= \Vinf^{-1} + 8 \rho \vect{u} \vect{u}^\top.
\end{align}

We restrict the supremum to $\rho \leq \rho_{max} = \frac{1}{8}(\max \Sp(\Vinf))^{-1}$ as in the proof of Lemma \ref{l:geqR}, as it is sufficient to get the upper bound \eqref{eq:leqminusR}. This enables the geometric series calculation as in the proof of Lemma \ref{l:geqR}. We get:
\begin{align}
\vect{u}^\top \mSigma_\rho \vect{u} &= \frac{ \vect{u}^\top \Vinf \vect{u} }{1 + 8 \rho \vect{u}^\top \Vinf \vect{u} } \\
\rho^* &:= \frac{1}{8} \frac{R}{\vect{u}^\top \Vinf \vect{u} (\vect{u}^\top \Vinf \vect{u} - R)} \\
\Sigma_{\rho^*} &= \Vinf \left(\id - R \frac{ S S^\top \Vinf}{(S^\top \Vinf S)^2 } \right) \\
H^*(R) &= - \frac{1}{8} \left( \frac{R}{S^\top \Vinf S } + \log \det \left(\id - R \frac{ S S^\top \Vinf}{(S^\top \Vinf S)^2 } \right) \right). \qedhere
\end{align}
\end{proof}

For completeness, we state an interesting corollary.

\begin{corollary}
\label{l:nonsymmetric_bound}
For any vectors $S_1,S_2$ such that $\nT{\vect{u}_1} \leq 1$,$\nT{\vect{u}_2} \leq 1$, and for any $i,j \leq d$:

\[ \Prob{ |\vect{u}_1^\top (\II - \Vinf) \vect{u}_2| > 3 R} \leq 6 \exp( -T H(R))\]
\[ \Prob{ |\II^{ij} - \Vinf^{ij}| > 3 R} \leq 6 \exp( -T H(R)).\]

\end{corollary}

\begin{proof}

Denote $\Delta \mat{C} := \II - \Vinf$. We have 
\[|\vect{u}_1^\top \Delta \mat{C} \vect{u}_2| \leq \frac{1}{2} \left| (\vect{u}_1 + \vect{u}_2)^\top \Delta \mat{C} (\vect{u}_1 + \vect{u}_2) + \vect{u}_1^\top \Delta \mat{C} \vect{u}_1 + \vect{u}_2^\top \Delta \mat{C} \vect{u}_2 \right|.\]
Each time with probability at least $1 - \exp( -T H^*(R))$, we have $|\vect{u}_1^\top \Delta \mat{C} \vect{u}_1 | \leq R$, $|\vect{u}_2^\top \Delta \mat{C} \vect{u}_2 | \leq R$ and $|(\vect{u}_1+\vect{u}_2)^\top \Delta \mat{C} (\vect{u}_1+\vect{u}_2) | \leq 4 R$.

For the second inequality, apply the first with $\vect{u}_1$ and $\vect{u}_2$ set respectively as the $i$-th and $j$-th vectors of the canonical basis.

\end{proof}


\subsection{Proof of Theorem \ref{th:1norm_empirical_bound} (Lasso error bound)}
\label{ss:master_theorem}

\begin{lemma}
\label{l:al-kashi}
For any matrix $\mA $ and any $\lambda >0$, we have:

\beq{eq:alkashi}{ 
\nL{( \hAl - \optA ) X }^2 - \nL{ ( \mA  - \optA ) X }^2 \leq 2 \sF{ \IW, \mA  - \hAl } - \nL{ ( \mA - \hAl) X }^2 + 2 \lambda ( \nO{\mA } - \nO{ \hAl } ).
}
\end{lemma}

\begin{proof}
As $\cL_T(\mA) = \tr \mA^\top \IW - \frac{1}{2}(\mA - \optA) \II (\mA - \optA)^\top + \frac{1}{2} \optA \II \optA^\top $, the gradient is $\IW + (\mA - \optA)\II$. The optimality condition applied to $\hAl$ gives that there exists a $\mB$ in the sub-derivative of the $\ell^1$ norm computed at $\hAl$ such that $\IW + (\hAl - \optA)\II + \lambda \mB$. $\mB$ being in the sub-derivative, we have $\sF{\mB,\mA- \hAl} \leq \nO{\mA} - \nO{\hAl}$.

Applying this to the following formula and observing that for any matrix $\mM$, $\nL{\mM X}^2 = \sL{\mM^\top \mM, \II}$, we get:
\begin{align}
S &:= \nL{ ( \hAl - \optA ) X }^2 - \nL{ ( \mA  - \optA ) X }^2 + \nL{ ( \hAl - \mA  ) X }^2 \\
&= \sF{ \II, \left( \hAl - \optA \right)^\top ( \hAl - \optA ) - ( \mA  - \optA )^\top ( \mA  - \optA ) + ( \hAl - \mA  )^\top ( \hAl - \mA  ) } \\
&= 2 \sF{ \II,  ( \optA - \hAl )^\top ( \mA  - \hAl ) } \\
&= 2 \sF{ \IW + \lambda \mB, \mA  - \hAl } \\
&\leq 2 \sF{ \IW, \mA  - \hAl} + 2 \lambda ( \nO{\mA } - \nO{\hAl} ). \qedhere
\end{align}
\end{proof}

Observe the preceding is true for any value of $\lambda$. Choose then $\lambda$ as in \eqref{def:lambda}. We have for instance $\gamma^{-1}\lambda = \theta(x,X)$ with $x = \frac{1}{2} \log \frac{2 \pi^2 d^2}{3 \epsilon_0}$ and $\theta$ as in Equation \eqref{def:theta}. Second, we assume $T \geq T_1 := H \left(\frac{\kappa^2}{9 (c_0 + 2)^2}\right)^{-1} ( s \log (21d \wedge 21ed/s) + \log 4 \epsilon_0^{-1})$. Therefore, using Theorems \ref{th:deviation} and Corollary~\ref{c:re_upper_bound}, we have for any matrix $\mU$:
\begin{multline}
\label{eq:re_dev}
\Prob{\inf_{\vect{u} \in C(s,c_0)} \frac{\nL{\vect{u}^\top X}}{\nT{\vect{u}}} \geq \kappa \cap \sF{\mU,\IW} \leq \gamma^{-1} \lambda \nO{\mU} \cap \forall i, \kappa^2 \leq \diagII^{ii} \leq \Vinf^{ii} + \kappa^2} \\ \geq 1 - \epsilon_0.
\end{multline}

We proceed to the proof of Theorem \ref{th:1norm_empirical_bound}.

\begin{proof}
We assume for the all what follows that the observation falls in the set of events defined by Equation \eqref{eq:re_dev} where we take $c_0 := \frac{\gamma + \tau + 1}{\gamma - \tau - 1}$. Therefore the inequalities we prove hold with probability at least $1 - \epsilon_0$. Denote $\mU = \mA - \hAl$. From Lemma \ref{l:al-kashi}, we have
\begin{align}
S &:= 2 \tau \gamma^{-1} \lambda \nO{\mU} + \nL{ ( \hAl - \optA ) X }^2 - \nL{ \left( \mA  - \optA \right) X }^2 + \nL{ \mU X }^2  \\
&\leq 2 \tau \gamma^{-1} \lambda \nO{\mU} + 2 \sF{ \IW, \mU } + 2 \lambda ( \nO{\mA } - \nO{ \hAl } ) \\
&\leq  2 \lambda \left( (1+ \tau) \gamma^{-1} \nO{\mU} + \nO{\mA } - \nO{ \hAl } \right) \label{eq:ineq_notau}\\ 
&\leq 2 \lambda \sum_{i=1}^d (1+\tau)\gamma^{-1} \nO{\mU^{i, \bullet}} + \nO{\mA^{i, \bullet} } - \nO{ \ha_{\lambda}^{i, \bullet} } \\ 
&\leq 2 \lambda \sum_{\Delta^i > 0} \Delta^i \label{eq:bound_deltas}
\end{align}
where $\Delta^i := (1 + (1+\tau)\gamma^{-1}) \nO{\mU^{i, \bullet}_{\mid \cA^{i, \bullet}}} - (1 - (1+\tau)\gamma^{-1}) \nO{ \mU^{i, \bullet}_{\mid \bar{\cA}^{i, \bullet}}}$ and 
$\cA^{i, \bullet} = \supp \mA^{i, \bullet}$. This last inequality comes from Lemma \ref{l:norm_inequalities} where we use the fact that $(1+\tau) \gamma^{-1} < 1$. We only need to consider the indices $i$ such that $\Delta^i > 0$, for which
\begin{align}
\nO{\mU^{i, \bullet}_{\mid \bar{\cA}^{i, \bullet}}} &< \frac{1 + (1+\tau)\gamma^{-1}}{1 - (1+ \tau) \gamma^{-1}} \nO{\mU^{i, \bullet}_{\mid \cA^{i, \bullet}}} = c_0 \nO{\mU^{i, \bullet}_{\mid  \cA^{i, \bullet}}}\\
\nO{\mU^{i, \bullet}} &< (1 + c_0) \nO{\mU^{i, \bullet}_{\mid \cA^{i, \bullet}}} \leq (1 + c_0)\nO{\mU^{i, \bullet}_{\mid \cU^{i, \bullet}}}
\end{align}
where $\cU^{i, \bullet} = \supp \mU^{i, \bullet}$.

We have then $\mU^{i, \bullet} \in C(s,c_0)$. We apply the condition from Equation \eqref{eq:re_dev} and get $\Delta^i \leq \gamma^{-1}(\gamma + \tau + 1) \sqrt{s} \nT{\mU^{i, \bullet}} \leq \gamma^{-1}(\gamma + \tau + 1) \sqrt{s} \kappa^{-1} \nL{(\mU^{i, \bullet})^\top X}$. Observing that 
\[\sum_{\Delta^i > 0}  \nL{(\mU^{i, \bullet})^\top X} \leq \sum_{i = 1}^d \nL{(\mU^{i, \bullet})^\top X} \leq \sqrt{d} \nL{\mU X},\]
we get $S \leq 2 \lambda \gamma^{-1}(\gamma + \tau + 1) \sqrt{s} \kappa^{-1} \nL{ \mU X }$. Using
\[ 2 \lambda \gamma^{-1}(\gamma + \tau + 1) \sqrt{s} \kappa^{-1} \nL{ \mU X } - \nL{ \mU X }^2 \leq \left( \frac{\gamma + \tau + 1}{\gamma \kappa} \right)^2 \lambda^2 d s, \]
we conclude with Equation \eqref{eq:norm1_L2_err}.

\end{proof}

The proof of Corollary \ref{c:applied_inequalities} consists in using Theorem \ref{th:1norm_empirical_bound} with specific values of the parameters. We explicit it in the following proof.

\begin{proof}
\begin{enumerate}
\item It suffices to take $\tau = 0$ and $\mA = \optA$ in Equation \eqref{eq:norm1_L2_err}.
\item We take $\mA = \optA$ and $\tau >0$ in Equation \eqref{eq:norm1_L2_err}. We bound The $L^2$ norm from below by $0$ and get the result.
\item It suffices to apply Equation \eqref{eq:empirical_bound} with the Restricted Eigenvalue condition which states that $\nF{\hAl - \optA} \leq \kappa^{-1} \nL{(\hAl - \optA)X}$ as long as each line of $\hAl - \optA$ is in $C(s,c_0)$ (see Lemma \ref{l:L2_quadraticform}). To prove that last point, we fix an index $i \leq d$ and continue the proof from Equation \eqref{eq:bound_deltas}. Choose $\mA$ the matrix equal to $\hAl$ except on the $i$-th line, where we assume it is equal to $\optA$. Then $\mU$ is null except on the $i$-th line, where it is equal to the $i$-th line of $\optA - \hAl$. As we have already assumed $\tau = 0$, we get:

\begin{align}
2 \lambda \Delta^i &\geq \nL{ \left( \hAl - \optA \right) X }^2 - \nL{ \left( \mA  - \optA \right) X }^2 + \nL{ \mU X }^2 \\
&= \nL{ \left( \hAl - \optA \right) X }^2 - \nL{ \left( \hAl  - \optA + \vect{e}_i (\mU^{i, \bullet})^\top \right) X }^2 + \nL{ (\mU^{i, \bullet})^\top X }^2 \\
&\geq 2  \sL{(\optA - \hAl)X,\vect{e}_i (\mU^{i, \bullet})^\top X} = 2 \nL{ (\mU^{i, \bullet})^\top X }^2 \geq 0.
\end{align}

Where $\vect{e}_i$ is the $i$-th element of the canonical basis of $\R^d$. Using the same argument as in the proof of \ref{th:1norm_empirical_bound}, we conclude that for each $i$, $(\hAl - \optA)^{i, \bullet} \in C(s,c_0)$ which is what we wanted.

\item We apply the norm interpolation inequality: $\nQ{\mU}^q \leq \nO{\mU}^{2 - q} \nF{\mU}^{2q-2}$ to equations \eqref{eq:norm1_err} and \eqref{eq:norm2_err}.

\end{enumerate}
\end{proof}

We finish this Section with the statement of a few useful inequalities

\begin{lemma}
\label{l:norm_inequalities}
Take $\mA,\mB$ two $d \times d$ matrices, $\gamma \in (0,1)$ and denote $\mU = \mA-\mB$, $\cA := \supp \mA$. Then we have the following inequalities:

\begin{align}
\gamma \nO{\mU} + \nO{\mA} - \nO{\mB} &= \gamma \nO{\mU_{\mid \cA}} + \gamma \nO{\mU_{\mid \bar{\cA}}} + \nO{\mA_{\mid \cA}} - \nO{\mB_{\mid \cA}} - \nO{\mB_{\mid \bar{\cA}}} \\
&= \gamma \nO{\mU_{\mid \cA}} + \gamma \nO{\mU_{\mid \bar{\cA}}} + \nO{\mA_{\mid \cA}} - \nO{\mB_{\mid \cA}} - \nO{\mU_{\mid \bar{\cA}}} \\
&\leq (1+\gamma) \nO{ \mU_{\mid \cA }} - (1-\gamma) \nO{ \mU_{\mid \bar{\cA}} } \label{eq:ineq1}\\
&\leq (1+\gamma) \sqrt{\nZ{\mA}} \nF{ \mU_{\mid \cA }} - (1-\gamma) \nO{\mU_{\mid \bar{\cA}}}.
\end{align}

\end{lemma}


\subsection{Proof of Theorem \ref{th:sparse_lower_bound} (lower bound of the estimation error)}
\label{ss:lower_bound}

\begin{lemma}
\label{l:lower_bound_1}
For some constant $c>0$, there exists a set $\Omega_a \subset \{-1,0,1\}^{d\times d}$ such that for any $\mB \not = \mB' \in \Omega_a$:
\begin{enumerate}
\item $\mB$ is antisymmetric and its upper triangular section has non-negative entries
\item $\mB$ has at most $s-1$ non-zero entries by row
\item $\sum_{ij} \one{\mB^{ij} \not = \mB'^{ij}} \geq cds$
\end{enumerate}
and $\log |\Omega_a| \geq cds \log \frac{ced}{s}$.
\end{lemma}

\begin{proof}
Consider the sets of matrices $\Omega'_{\leq r}, \Omega'_{r}$ of antisymmetric matrices in $\{-1,0,1\}^{d \times d}$ such that the entries in the upper triangular section are all non-negative and with at most $r$ non-zero entries in each row for $\Omega'_{\leq r}$ and exactly $r$ non-zero entries in each row for $\Omega'_{r}$. For any $r \leq s-1$, $\Omega'_{r} \subset \Omega'_{\leq s-1}$. Fix then $r$ to be the largest even number smaller or equal to $s-1$: we have $s-2 \leq r \leq s-1 $. $\Omega'_{r}$ is one-to-one with the set $\mathrm{reg}(r,d)$, the set of $r$-regular labeled graphs with $d$ vertices. To see this, observe that applying the absolute value to a matrix in $\Omega'_{r}$ gives an adjacency matrix of a regular graph. The relation is one-to-one given assumption 1.

We know the asymptotic of $|\mathrm{reg}(r,d)|$. Take for example \cite{mckay:1991} and apply $k! = k^k e^{-k} \sqrt{2\pi k} \Psi(k)$ where $(12j+1)^{-1} < \log \Psi(k) < (12j)^{-1}$. We keep track only of the highest order on $d$, the relevant variable that is considered high.

\begin{align}
\log |\mathrm{reg}(r,d)| &\sim - \frac{r^2 - 1}{4} - \frac{r^3}{12d} + \log (dr)! - \log (dr/2)! - \frac{dr}{2} \log 2 - d \log r! \\
&\sim - \frac{r^2 - 1}{4} - \frac{r^3}{12d} + \frac{dr}{2} \log \frac{ed}{r} - \frac{d}{2} \log r -d \log \sqrt{2 \pi} + \frac{1}{2} \log 2 - \frac{1}{12dr} - \frac{d}{12r}.
\end{align}

Keeping only the highest order in $d$ gives $\log |\mathrm{reg}(r,d)| \sim \frac{dr}{2} \log \frac{ed}{r}$. We know due to the Erd\H{o}s-Gallai theorem \cite{erdos:1960} that the number of $r$-regular graphs is non-zero for $d >r$ as $r$ is chosen to be even. Hence there exists a constant $2c >0$ such that $|\mathrm{reg}(r,d)| \geq 2c d (r+2) \log ed/r$.

We continue by applying the Gilbert–Varshamov bound. For this, we need to compute the maximum volume of a Hamming ball of radius $K$, where $K$ is an integer: we fix some $\mA$ and will count the number of $\mA'$ that differ from $\mA$ on a maximum of $K$ entries. That volume is bounded by the one in the larger space of matrices with at most $dr$ non-zero entries. We have thus:

\[ V \leq \sum_{i=1}^K \binom{d^2}{i} \leq \sum_{i=1}^K \frac{K^i}{i!} \left( \frac{d^2}{K} \right)^i \leq \left( \frac{ed^2}{K} \right)^K. \]

Choose then $K = \lfloor c d (r+2) \rfloor \geq \lfloor cds \rfloor $. As $x \mapsto (ed^2/x)^x$ is increasing for $x \leq d^2$, we have $V \leq \left( \frac{ed}{c(r+2)} \right)^{cd(r+2)} \leq \left( \frac{ed}{cr} \right)^{cd(r+2)}$. The Gilbert-Varshamov bound gives us the existence of a set $\Omega_a \subset \Omega'_{r}$, verifying condition 3 and its size is at least:

\[ \log |\Omega_a| \geq \log |\Omega_r'| - \log V = cd(r+2) \log \frac{ced}{r} \geq cds \log \frac{ced}{s} . \]
\end{proof}

Lemma \ref{l:diag_antisymmetric} gives a way to construct a family of drift parameters with corresponding diagonal stationary covariance from a family of antisymmetric matrices.
\begin{lemma}
\label{l:diag_antisymmetric}
Take $\mA = \alpha \id + \mB$ for $\alpha>0$ and $\mB$ an antisymmetric matrix. Then we have:

\[ \Vinf(\mA) = \int_0^\infty e^{-\mA t} e^{-\mA^\top t} \dt = \frac{1}{2 \alpha} \id. \]
\end{lemma}

\begin{proof}
We have $\mB^\top = - \mB$ hence $i\mB$ is Hermitian and therefore unitarily diagonalizable. There exists an unitary matrix $\mU$ such that $\mB = i \mU \mD \mU^*$ where $\mD$ is a diagonal real matrix. Then $e^{-(\alpha \id + \mB)t} = e^{-\alpha t} \mU e^{- i \mD t} \mU^*$ and $e^{-(\alpha \id + \mB)^\top t} = e^{-(\alpha \id - \mB) t} = e^{-\alpha t} \mU e^{i \mD t} \mU^*$ hence $ \Vinf(\mA) = \int_0^\infty e^{-2 \alpha t} \dt \ \id = \id / (2\alpha).$
\end{proof}

\begin{corollary}
For some constant $c>0$ and $0 < \alpha <1/8$, there exists a set $\Omega$ with $\log |\Omega| \geq cds \log \frac{cd}{es}$ such that for any $\mA \not = \mA' \in \Omega$:
\begin{enumerate}
\item $\mA$ is row-$s$-sparse
\item $\Vinf(\mA) = \id$
\item $KL(\dP_\mA,\dP_{\mA'}) \leq \alpha \log |\Omega| $
\item $\nF{\mA - \mA'}^2 \geq \alpha c^2 T^{-1} ds \log \frac{ced}{s}.$
\end{enumerate}

\end{corollary}

\begin{proof}
Define $\Omega = \{ \frac{1}{2} \id + w \mB: \mB \in \Omega_a \}$ for some $w>0$ and $\Omega_a$ as defined in Lemma \ref{l:lower_bound_1}. Then $|\Omega| = |\Omega_a|$ and hence $\log |\Omega| \geq cds \log \frac{ced}{s}$. Condition 1 is verified trivially and Lemma \ref{l:diag_antisymmetric} gives condition 2. Further, from Lemma \ref{l:lower_bound_1} point 3. $\nF{\mA - \mA'}^2 = w^2 \nF{\mB - \mB'} ^2\geq w^2 cds$. Also, the maximum Hamming distance being $2ds$, we get $\nF{\mA - \mA'}^2 \leq 2 w^2 ds$.

The Kullback-Leibler divergence writes $KL(\dP_\mA, \dP_{\mA'}) = \Espq{\mA}{\log \frac{\dd \dP_\mA}{\dd \dP_{\mA'}}} = \frac{T}{2} \tr (\mA' - \mA) \Vinf(\mA) (\mA' - \mA)^\top$. Using condition 2, $KL(\dP_\mA, \dP_{\mA'}) = \frac{T}{2} \nF{\mA - \mA'}^2 \leq w^2 Tds.$


Choose $0 < \alpha < 1/8$ and $w^2 = \alpha c T^{-1} \log \frac{ced}{s}$ such that $KL(\dP_\mA,\dP{\mA'}) \leq \alpha \log |\Omega|$. Then we also have $\nF{\mA - \mA'}^2 \geq \alpha c^2T^{-1}ds \log \frac{ced}{s}.$
\end{proof}

Theorem \ref{th:sparse_lower_bound} is the corollary of the preceding and of Theorem 2.7 from \cite{tsybakov:2008}.


\subsection{Proof of Theorem \ref{th:RE} (Restricted Eigenvalue property)}
\label{ss:RE}

\begin{lemma}
\label{l:sparse_proba_bound}

Take a random symmetric matrix $\mC$. Define $K(s) := \{ \vect{u} : \nZ{\vect{u}} \leq s\}$. Assume that for any vector $\vect{u} \in K(s)$ such that $\nT{\vect{u}} \leq 1$, and any $R \geq 0$, we have $\Prob{|\vect{u}^\top \mC \vect{u}| \geq R} \leq p(R)$. Then:

\[ \Prob{\sup_{\vect{u} \in K(s), \nT{\vect{u}} \leq 1} |\vect{u}^\top \mC \vect{u}| \geq 3R} \leq 21^s (d^{s} \wedge (ed/s)^s) p(R). \]

\end{lemma}

We omit the proof as it follows exactly the same steps as the one of Lemma F.2 from the supplement to \cite{basu:15}. Recall now the definition $C(s,c_0) = \{ \vect{u} : \nO{\vect{u}} \leq (1 + c_0) \nO{\vect{u}_{\cI_s(\vect{u}))}} \}$, where $\cI_s(\vect{u}))$ is the set of indices of the $s$ largest values of $|\vect{u}|$. Using Lemmas F.1 and F.3 from the supplement to \cite{basu:15}, we get that:

\beq{eq:sup_cone_s_sparse}{
\sup_{\vect{u} \in C(s,c_0), \nT{\vect{u}} \leq1} \vect{u}^\top \II \vect{u} \leq 3 (c_0 + 2)^2 \sup_{\vect{u} \in K(2s), \nT{\vect{u}} \leq 1} \vect{u}^\top \II \vect{u}.
}

Taking this combined with Lemma \ref{l:sparse_proba_bound} applied using hypothesis (H4), we get the following Lemma.

\begin{lemma}
\label{l:ii_deviation}

For any $R \geq 0$, we have:

\[ \Prob{\sup_{\substack{\vect{u} \in C(s,c_0) \\ \nT{\vect{u}} \leq 1}} |\vect{u}^\top (\II - \Vinf) \vect{u}| \geq R } \leq 2 \exp \left(-T H \left(\frac{R}{9 (c_0 + 2)^2}\right) + s \log (21d \wedge 21ed/s)\right). \]

\end{lemma}

We conclude by proving the Restricted Eigenvalue inequality.

\begin{theo}[Restricted Eigenvalue]
\label{th:RE_proba}

\[ \Prob{\inf_{\substack{\vect{u} \in C(s,c_0) \\ \nT{\vect{u}} \leq 1}} |\vect{u}^\top \II \vect{u}| \leq \kappa^2 } \leq 2 \exp \left(-T H \left(\frac{\kappa^2}{9 (c_0 + 2)^2}\right) + s \log (21d \wedge 21ed/s)\right). \]
\end{theo}

\begin{proof}
We apply the lemma with $R = \kappa^2 = \min \Sp(\Vinf)/2$ and use the fact that 
\[ \inf_{\vect{u} \in C(s,c_0), \nT{\vect{u}} = 1} \vect{u}^\top \Vinf \vect{u} \geq \min \Sp(\Vinf). \qedhere\]
\end{proof}

We can actually have, with the same probability, an additional upper bound on $\vect{u}^\top \II \vect{u}$, from which we get a bound on the diagonal elements of $\II$, as stated in
\begin{corollary}
\label{c:re_upper_bound}

Set $\epsilon_0 \in (0,1)$. For $T \geq T_0 := H \left(\frac{\kappa^2}{9 (c_0 + 2)^2}\right)^{-1} ( s \log (21d \wedge 21ed/s) + \log 4 \epsilon_0^{-1})$, we have:

\beq{eq:delta_0_bound}{
\Prob{\inf_{\vect{u} \in C(s,c_0), \nT{\vect{u}} \leq 1} |\vect{u}^\top \II \vect{u}| \geq \kappa^2 , \nInf{ \diag \II} \leq \nInf{\diag \Vinf} + \kappa^2} \geq 1 - \frac{\epsilon_0}{2}.
}
\end{corollary}

\begin{proof}
From Lemma \ref{l:ii_deviation} we get a set of events of probability $1- \frac{\epsilon_0}{2}$ which verifies $\sup_{\substack{\vect{u} \in C(s,c_0) \\ \nT{\vect{u}} \leq 1}} |\vect{u}^\top (\II - \Vinf) \vect{u}| \leq R$. From this follows the infimum condition as in Theorem \ref{th:RE_proba}. Further, by taking for some $i$, $\vect{u} = \vect{e}_i \in C(s,c_0)$, we get $|\II^{ii} - \Vinf^{ii}| \leq \kappa^2$ and the result follows.
\end{proof}

We finish this Section by showing different ways the Restricted Eigenvalue property can be expressed, using matrices, vectors, $\II$ or $L^2$ norm. Using this we get for instance Theorem \ref{th:RE}.

\begin{lemma}
\label{l:L2_quadraticform}
For any subset $E \subset \R^d$, we have:

\begin{align}
\sup_{\mA : \forall i \leq d, \mA^{i, \bullet} \in E} \frac{\nL{\mA X}}{\nF{\mA}} &= \left( \sup_{\mA: \forall i \leq d, \mA^{i, \bullet} \in E, \nF{\mA} \leq 1} \tr \mA \II \mA^\top \right)^{1/2} \\
&= \left( \sup_{\vect{u} \in E, \nT{\vect{u}} \leq 1} \vect{u}^\top \II \vect{u} \right)^{1/2} \\
&= \sup_{\vect{u} \in E} \frac{\nL{\vect{u}^\top X}}{\nT{\vect{u}}}.
\end{align}

\begin{align}
\inf_{\mA : \forall i \leq d, \mA^{i, \bullet} \in E} \frac{\nL{\mA X}}{\nF{\mA}} &= \left( \inf_{\mA: \forall i \leq d, \mA^{i, \bullet} \in E, \nF{\mA} \leq 1} \tr \mA \II \mA^\top \right)^{1/2} \\
&= \left( \inf_{\vect{u} \in E, \nT{\vect{u}} \leq 1} \vect{u}^\top \II \vect{u} \right)^{1/2} \\
&= \inf_{\vect{u} \in E} \frac{\nL{\vect{u}^\top X}}{\nT{\vect{u}}}.
\end{align}

\end{lemma}

\begin{proof}

We have the following relations for any matrix $\mA$:

\beq{eq:norms_equalities}{
\nL{\mA X}^2 = \tr \mA \II \mA^\top = \sum_{i = 1}^d (\mA^{i, \bullet})^\top \II (\mA^{i, \bullet}) =\sum_{i=1}^d \nT{\mA^{i, \bullet}}^2 \left( \frac{\mA^{i, \bullet}}{\nT{\mA^{i, \bullet}}} \right)^\top \II \left( \frac{\mA^{i, \bullet}}{\nT{\mA^{i, \bullet}}} \right)
}
Similarly, for a vector $\vect{u}$, $\nL{\vect{u}^\top X}^2 = \vect{u}^\top \II \vect{u}$. 

Assume that for any $i \leq d$, $\mA^{i, \bullet} \in E$. We immediately get $\nL{\mA X}^2 \leq \nF{\mA}^2 \sup_{\vect{u} \in E} \frac{\nL{\vect{u}^\top X}^2}{\nT{\vect{u}}^2}$. Hence we have $\sup_{\mA : \forall i \leq d, \mA^{i, \bullet} \in E} \frac{\nL{\mA X}}{\nF{\mA}} \leq \sup_{\vect{u} \in E} \frac{\nL{\vect{u}^\top X}}{\nT{\vect{u}}}$. Choose now a vector $\vect{u}$ that realizes the supremum on the RHS. By choosing $\mA = \one{}\vect{u}^\top$, we get the equality in the inequality.

The proof for the infimums is exactly analogous.

\end{proof}


\subsection{Proof of Theorem \ref{th:adalasso_oracle} (asymptotic properties of the Adaptive Lasso)}
\label{ss:asymptotic}

$\hAla$ is defined as the minimizer of the penalized log-likelihood, see Equation \eqref{def:adalasso}. The penalization includes the MLE and we denote $\mat{\Gamma} = 1/|\hA|^\gamma$. We start by re-centering and changing the normalization of the objective function, then separating the log-likelihood from the penalization:

\begin{align}
\sqrt{T} ( \hAla - \optA ) = \hat{\mU} &= \argmin_\mU \phi_1(\mU) + \phi_2(\mU) \\
\phi_1(\mU) &:= T \cL_T(\optA + \mU/\sqrt{T}) - T \cL_T(\optA) \\
\phi_2(\mU) &:= \lambda T \nO{ (\optA + \mU/\sqrt{T}) \odot \mat{\Gamma} } - \lambda T \nO{ \optA \odot \mat{\Gamma} }.
\end{align}

Using equation \eqref{eq:logL_matrix}, we characterize the limit structure of $\phi_1$.
\begin{align}
\phi_1(\mU) &= \sqrt{T} \tr \mU^\top \IW + \frac{1}{2} \tr \mU \II \mU^\top \\
&= \frac{1}{\sqrt{T}}\int_0^T (\mU X_t)^\top \dW_t + \frac{1}{2T} \int_0^T \nT{\mU X_t}^2 \dt.
\end{align}

\begin{enumerate}
\item From assumption (H1), $X$ is ergodic and therefore we can apply the ergodic theorem for the classical integral:

\beq{eq:class_integ_limit}{
\frac{1}{T} \int_0^T \nT{\mU X_t} \dt \xrightarrow{\dP} \Esp{\nT{\mU X_0}^2} = \tr \mU \Vinf \mU^\top.
}
\item The stochastic integral $M_T = \int_0^T (\mU X_t)^\top \dW_t$ is a martingale, for which we apply the central limit theorem for martingales, recalled below in Lemma \ref{l:TCL_martingale}.

\begin{lemma}[{\cite[Theorem 4.1]{vanzanten:2000}}]
\label{l:TCL_martingale}
Let $(M_t;\cF_t: t \geq 0)$ be a $d$-dimensional continuous local martingale. If there exist invertible, non-random $d \times d$-matrices $(K_t : t \geq 0)$ such that as $t \rightarrow \infty$
\begin{itemize}
\item $K_t \langle M \rangle_t K_t^\top \xrightarrow{\dP}   \eta \eta^\top$ where $\eta$ is a random $d \times d$-matrix; 
\item $|K_t| \rightarrow 0$;
\end{itemize}
then, for each $\R^k$-valued random vector $X$ defined on the same probability space as $M$, we have 
\[(K_t M_t,X) \xrightarrow{\rm d} (\eta Z,X) \qquad \text{as} \quad t \rightarrow \infty,\]
where $Z \eqL \cNN{0,\id}$ and $Z$ is independent of $(\eta,X)$.
\end{lemma}
We have:
\begin{align}
\llangle{M}_T &= \int_0^T \nT{\mU X_t}^2 \dt \\
\left( \frac{1}{\sqrt{T}} \right)^2 \llangle{M}_T &\xrightarrow{\dP} \tr \mU \Vinf \mU^\top.
\end{align}
Hence
\beq{eq:stoch_integ_limit}{
\frac{1}{\sqrt{T}} \int_0^T (\mU X_t)^\top \dW_t \xrightarrow{\cL} \cNN{0,\tr \mU \Vinf \mU^\top}.
}

Introduce then a centered Gaussian $d \times d$ matrix $\mG$ such that $\cov{\mG^{ij},\mG^{kl}} = \one{j=l} \Vinf^{ik}$. Then, for any matrix $\mU$:

\begin{itemize}
\item $\tr \mU \mG$ is a Gaussian variable
\item $\Esp{\tr \mU \mG} = 0$
\item $\Var{\tr \mU \mG} = \sum_{ijkl} \mU^{ji} \mU^{lk} \Cov{\mG^{ij} \mG^{kl}} = \sum_{ijk} \mU^{jk} \mU^{ji} \Vinf^{ik} = \tr \mU \Vinf \mU^\top$
\end{itemize}
From there,
\beq{eq:stoch_integ_limit}{
\frac{1}{\sqrt{T}} \int_0^T (\mU X_t)^\top \dW_t \xrightarrow{\cL} \tr \mU \mG.
}

\end{enumerate}

From the two preceding points, we conclude:

\beq{eq:phi1_limit}{
\phi_1(\mU) \xrightarrow{\cL} \tr \left( \frac{1}{2} \mU \Vinf \mU^\top + \mU \mG \right).
}

Second, we observe the limit structure of the penalization $\phi_2$. Denote $\optcA =  \supp \optA$. We have $\phi_2(\mU) = \lambda T \sum_{ij} \mat{\Gamma}^{ij} \left( \left|\optA^{ij} + \mU^{ij}/\sqrt{T} \right| -  \left| \optA^{ij} \right|\right)$.

\begin{enumerate}
\item If $(i,j) \in \optcA$, for high enough $T$, $\sqrt{T}\left|\optA^{ij} + \mU^{ij}/\sqrt{T} \right| -  \left| \optA^{ij} \right| = \mathrm{sign}(\optA^{ij}) |\mU^{ij}|$, and $\mat{\Gamma}^{ij} \xrightarrow{\dP} |\optA^{ij}|^{-\gamma}$, a positive constant. From our assumption, $\lambda \sqrt{T} \to 0$, hence 
\[ \lambda T \mat{\Gamma}^{ij} \left( \left|\optA^{ij} + \mU^{ij}/\sqrt{T} \right| -  \left| \optA^{ij} \right|\right) \xrightarrow{\dP} 0.\]

\item Else, $(i,j) \in \bar{\optcA}$ and then $\lambda T \mat{\Gamma}^{ij} \left( \left|\optA^{ij} + \mU^{ij}/\sqrt{T} \right| -  \left| \optA^{ij} \right|\right) = \lambda T^{(\gamma+1)/2}| \sqrt{T} \hA^{ij}|^{-\gamma} \left|\mU^{ij}\right|$. 
We know the MLE is root-$t$ consistent, hence $\sqrt{T} \hA^{ij} = \OO{1}$ and by assumption $\lambda T^{(\gamma+1)/2} \to +\infty$. Hence, the expression diverges to $+\infty$.
\end{enumerate}

For high $T$, $\phi_2$ becomes flat $0$ on the support and infinite outside. Combining with the result from Equation \eqref{eq:phi1_limit}, we have:

\beq{eq:Phi_limit}{
\phi_1(\mU) + \phi_2(\mU) \xrightarrow{\cL} \begin{cases} +\infty & \text{if } \mU_{\bar{\optcA}} = 0, \\ \tr \left( \frac{1}{2} \mU \Vinf \mU^\top + \mU G \right) &\text{else.}\end{cases}
}

We finally need to compute the minimum of that function. Take $\mU$ such that $\mU_{\bar{\optcA}} = 0$. Recall that we treat a matrix restricted to a set of indices as a vector. Then:

\begin{align}
\tr \mU \mG &= \sum_{ij} \mU^{ij} \mG^{ji} \\
&= \left((\mG^\top)_{\mid \optcA}\right)^\top \mU_{\mid \optcA} \\
\tr \mU \Vinf \mU^\top &= \sum_{ijkl} \mU^{ij} \mU^{kl} (\one{i=k} \Vinf^{jl}) \\
&= (\mU_{\mid \optcA})^\top  (\Vinf \otimes \id)_{\mid \optcA^2} \mU_{\mid \optcA}.
\end{align}

$(\Vinf \otimes \id)_{\mid \optcA^2}$ is the restriction of $\Vinf \otimes \id$ to the indices in $\optcA^2 := \optcA \times \optcA$ and $\Vinf \otimes \id$ is invertible and symmetric, with inverse $\Vinf^{-1} \otimes \id$, hence $(\Vinf \otimes \id)_{\optcA^2}$ is invertible and symmetric. Hence, $\tr \left( \frac{1}{2}\mU \Vinf \mU^\top + \mU \mG \right) = \frac{1}{2}(\mU_{\mid \optcA})^\top  (\id \otimes \Vinf)_{\mid \optcA^2} \mU_{\optcA}+ (\mU_{\mid \optcA})^\top (\mG^\top)_{\mid \optcA}$, which is a quadratic function of $\mU_{\mid \optcA}$ and we compute easily the minimum, which shows that $\widehat \mU_{\mid \optcA}$ is a centered Gaussian and completes the proof of point 2:
\beq{eq:hatU_sol}{
\widehat \mU_{\mid \optcA} \xrightarrow{\cL} - (\mG^\top)_{\mid \optcA} \left((\Vinf \otimes \id)_{\mid \optcA^2}\right)^{-1}, \quad \widehat \mU_{\bar{\optcA}} \xrightarrow{\cL} 0
}
and we find that $ (\mG^\top)_{\mid \optcA} \left((\Vinf \otimes \id)_{\mid \optcA^2}\right)^{-1} \sim \cNN{0,\mathcal{V}}$ with $\mathcal{V} := \left((\Vinf \otimes \id)_{\mid \optcA^2}\right)^{-1}.$

Proceed now with point 1. We proved in the preceding the asymptotic normality of the convergence on $\optcA$, from which we deduce $\forall (i,j) \in \optcA, \Prob{\hAla^{ij} \not = 0} \to 1$. Take now $(i,j) \in \bar{\optcA}$ and assume the event $\hAla^{ij} \not = 0$. We write the optimality conditions, multiplied by $\sqrt{T}$, and apply the absolute value: $\left|\sqrt{T}S^{ij} + (\sqrt{T}\hAla \II)^{ij} \right| = \lambda \mat{\Gamma}^{ij} \sqrt{T}$. When $T \to +\infty$:

\begin{align}
\sqrt{T}S^{ij} = \frac{1}{\sqrt{T}} \int_0^T X_t^j \dW_t^i &\xrightarrow{\cL} \cNN{0,\Vinf^{jj}} \\
\sqrt{T}\hAla &\xrightarrow{\cL} \cNN{0,\cal{V}} \\
\II &\xrightarrow{\dP}\Vinf \\
\lambda \mat{\Gamma}^{ij} \sqrt{T} = \lambda T^{(\gamma+1)/2} (\sqrt{T} |\hA^{ij}|)^{-\gamma} &\xrightarrow{\dP} +\infty.
\end{align}

When $T \to +\infty$, we can therefore bound the probability of $\hAla^{ij} \not = 0$ by the probability that the sum of some two Gaussians is equal to a diverging number, in absolute value. This is clearly of a probability converging to zero. Therefore, $\forall (i,j) \in \optcA, \Prob{\hAla^{ij} = 0} \to 1$.


\subsection{Deviation bound}
\label{ss:deviation}

Recall Bernstein's inequality, see Chapter 4, Exercise 3.16 in \cite{revuzyor:99}:

\begin{lemma}[Bernstein's inequality]
Let $M$ be a scalar continuous local martingale. For all $a>0, b>0$:
\beq{eq:hoeffding}{
\Prob{ M_t \geq a, \langle M \rangle_t \leq b } \leq \exp \left( - \frac{a^2}{2b} \right).
}
\end{lemma}

\begin{lemma}
\label{l:martingale_loglog_e}
Let M be a scalar continuous local martingale. For any $x > 0$:

\beq{eq:martingale_loglog_e}{ \Prob{ M_t \geq \sqrt{4e \langle M \rangle_t \left( x + 
\log \left( 2 + \left| \log \langle M \rangle_t \right| \right) \right)}} \leq \frac{\pi^2}{3} \exp(-2x).}
\end{lemma}

\begin{proof}

Observe that if $j \leq \log \langle M \rangle_t \leq j + 1$ for some integer $j$, then $| \log \langle M \rangle_t | \geq |j| -1$.

\begin{align}
P &= \Prob{ M_t \geq \sqrt{4e \langle M \rangle_t \left( x + \log \left( 2 + \left| \log \langle M \rangle_t \right| \right) \right)}} \\
&= \sum_{j \in \Z} \Prob{ M_t \geq \sqrt{4e \langle M \rangle_t \left( x + \log \left( 2 + \left| \log \langle M \rangle_t \right| \right) \right)}, e^j \leq \langle M \rangle_t < e^{j+1} } \\
&\leq \sum_{j \in \Z} \Prob{ M_t \geq \sqrt{4 e^{j+1} \left( x + \log \left( 1 + \left| j \right| \right) \right)}, \langle M \rangle_t < e^{j+1} } \\
& \leq \sum_{j \in \Z} \exp \left( - 2 \left( x + \log \left( 1 + \left| j \right| \right) \right) \right) \\
&= \exp(- 2 x) \sum_{j \in \Z} \frac{1}{\left( 1+|j|\right)^2} \\
&= 2 \exp(- 2 x) \sum_{j \in \N^*} \frac{1}{ j^2} \\
&= \frac{\pi^2}{3} \exp(-2x).
\end{align}

\end{proof}


\begin{theo}
\label{th:deviation}


Define for $x>0$:

\beq{def:theta}{
\theta(x,(X_t)) := \sqrt{4 e T^{-1} |\diag(\II)|_\infty \left( x + \log ( 2 + | \log T \diag(\II) |_\infty ) \right) },
}

where we denote $\diag$ the extraction of the diagonal of a matrix and $\log$ applies naturally to each term (which are all positive).

For any matrix $\mU$, the set of events

\beq{def:cond_C1}{
\Prob{ \sFl{\mU,T^{-1} \int_0^T \dW_t X_t^\top} \leq \theta(x,(X_t)) \nO{\mU} } \geq 1 - \frac{\pi^2}{3} \exp (- 2 x + 2 \log d).
}

\end{theo}

\begin{proof}



Set $i,j \leq d$. Recall that $\int_0^T \dW_s^i X_t^j$ is a martingale, and its bracket is $ \int_0^T (X_t^j)^2 \dt = T\II^{jj}$. By applying Lemma \ref{l:martingale_loglog_e}:


\beq{eq:f_prob_bound}{
\Prob{\int_0^T \dW_s^i X_t^j \geq \sqrt{4e T\II^{jj} \left( x + 
\log \left( 2 + \left| \log T\II^{jj} \right| \right) \right)}} \leq \frac{\pi^2}{3} \exp(-2x).
}

We have $\sqrt{4e T\II^{jj} \left( x + 
\log \left( 2 + \left| \log T\II^{jj} \right| \right) \right)} \leq T \theta(x,(X_t))$, hence using an union bound:


\beq{eq:inf_bound}
{\Prob{ \nInflarge{\int_0^T \dW_s X_t^\top} \geq T \theta(x,(X_t))} \leq \frac{\pi^2}{3} d^2 \exp(-2x).}

Observe now that by homogeneity, it suffices to prove Equation \eqref{def:cond_C1} for any matrix $\mU$ such that $\theta(x,(X_t)) \nO{\mU} \leq 1$. Then we have: 
\beq{eq:2norm_1norm_infnorm}{
\left| \sFl{\mU,T^{-1} \int_0^T \dW_t X_t^\top} \right| = \left| \sum_{ij} \theta(x,(X_t)) \mU^{ij} \frac{\int_0^T \dW_t^i X_t^j}{T \theta(x,(X_t))} \right| \leq \frac{\nInf{\int_0^T \dW_s X_t^\top}}{T \theta(x,(X_t))}.
}

Equation \eqref{def:cond_C1} follows from Equation \eqref{eq:inf_bound}.

\end{proof}

{
\bibliographystyle{alpha}
\bibliography{SGM}
}

\end{document}